\documentclass{article}

\usepackage{microtype}
\usepackage{graphicx}
\usepackage{subcaption}
\usepackage{booktabs} 

\usepackage{hyperref}
\usepackage{dsfont}

\newcommand{\NN}{\mathbb{N}}

\newcommand{\RR}{\mathbb{R}}
\newcommand{\CC}{\mathbb{C}}
\newcommand{\U}{\mathbb{U}}
\newcommand{\E}{\mathbb{E}}
\newcommand{\innerproduct}[2]{\langle #1, #2 \rangle}

\newcommand{\cN}{\mathcal{N}}

\newcommand{\cU}{\mathcal{U}}

\newcommand{\bfone}{\mathds{1}}

\usepackage[accepted]{icml2024}

\usepackage{amsmath}
\usepackage{amssymb}
\usepackage{mathtools}
\usepackage{amsthm}

\usepackage[capitalize,noabbrev]{cleveref}

\usepackage{centernot}
\DeclareMathOperator*{\im}{im}
\DeclareMathOperator*{\Unif}{Unif}
\DeclareMathOperator*{\var}{var}
\DeclareMathOperator*{\cov}{cov}
\DeclareMathOperator*{\dist}{dist}
\DeclareMathOperator*{\diag}{diag}
\DeclareMathOperator*{\Gram}{Gram}
\DeclareMathOperator*{\argmin}{arg\,min}

\usepackage{comment}
\usepackage{placeins}

\theoremstyle{plain}
\newtheorem{theorem}{Theorem}[section]
\newtheorem{proposition}[theorem]{Proposition}
\newtheorem{lemma}[theorem]{Lemma}
\newtheorem{corollary}[theorem]{Corollary}
\theoremstyle{definition}

\newtheorem{assumption}[theorem]{Assumption}
\newtheorem{example}[theorem]{Example}
\theoremstyle{remark}
\newtheorem{remark}[theorem]{Remark}

\renewcommand{\leq}{\leqslant}
\renewcommand{\geq}{\geqslant}

\usepackage[textsize=tiny]{todonotes}

\icmltitlerunning{On the Minimal Degree Bias in Generalization on the Unseen for non-Boolean Functions}

\begin{document}

\twocolumn[
\icmltitle{On the Minimal Degree Bias in Generalization on the Unseen\\ for non-Boolean Functions}



\icmlsetsymbol{equal}{*}

\begin{icmlauthorlist}
\icmlauthor{Denys Pushkin}{epfl}
\icmlauthor{Rapha\"el Berthier}{inria}
\icmlauthor{Emmanuel Abbe}{epfl,apple}
\end{icmlauthorlist}

\icmlaffiliation{epfl}{School of Computer and Communication Sciences, EPFL, Lausanne, Switzerland}
\icmlaffiliation{apple}{Apple, Machine Learning Research (MLR), Switzerland}
\icmlaffiliation{inria}{Inria Sorbonne Universit\'e, Paris, France. (While RB is currently affiliated with Inria, the work presented here was partly done while affiliated with the School of Computer and Communication Sciences, EPFL, Lausanne, Switzerland)}

\icmlcorrespondingauthor{Denys Pushkin}{denys.pushkin@epfl.ch}
\icmlcorrespondingauthor{Raphael Berthier}{raphael.berthier@inria.fr}
\icmlcorrespondingauthor{Emmanuel Abbe}{emmanuel.abbe@epfl.ch}

\icmlkeywords{Machine Learning, ICML}

\vskip 0.3in
]



\printAffiliationsAndNotice{}  



\begin{abstract}
We investigate the out-of-domain generalization of random feature (RF) models and Transformers. We first prove that in the `generalization on the unseen (GOTU)' setting, where training data is fully seen in some part of the domain but testing is made on another part, and for RF models in the small feature regime, the convergence takes place to interpolators of minimal degree as in the Boolean case \cite{gotu}. We then consider the sparse target regime and explain how this regime relates to the small feature regime, but with a different regularization term that can alter the picture in the non-Boolean case. 
We show two different outcomes for the sparse regime with q-ary data tokens: (1) if the data is embedded with roots of unities, then a min-degree interpolator is learned like in the Boolean case for RF models, (2) if the data is not embedded as such, e.g., simply as integers, then RF models and Transformers may not learn minimal degree interpolators. This shows that the Boolean setting and its roots of unities generalization are special cases where the minimal degree interpolator offers a rare characterization of how learning takes place. For more general integer and real-valued settings, a more nuanced picture remains to be fully characterized.
\end{abstract}

\section{Introduction}
Some of the most challenging tasks for state-of-the-art machine learning models reside in settings where the training data is not representative of the testing data, or more specifically, when there is significant distribution shift. This is in particular central to `reasoning tasks', such as arithmetic and algebra \cite{saxton2019analysing, lewkowycz2022solving-minerva}, visual reasoning such as CLEVR \cite{johnson2017clevr}, physical reasoning such as Phyre \cite{bakhtin2019phyre}, algorithmic data such as CLRS \cite{velivckovic2022clrs} and reasoning on graphs \cite{mahdavi2022better}. In such settings, the combinatorial nature of the data makes comprehensive data sampling challenging, resulting naturally in `holes' in the sampled domain. 

An archetype example of this kind is the `length generalization' setting: no matter how dense we sample discrete inputs, the training data will have inputs of some bounded length, and one can naturally ask for generalization to larger length.
This has motivated \citet{gotu} to consider a special case of out-of-distribution generalization: generalization on the unseen (GOTU). In its most extremal form, the GOTU setting assumes that part of the domain (in a large embedded dimension) is fully observed at training, and the generalization of the model is tested on a new (unseen) part of the domain. Therefore in this setting, there is no `estimation error' since the model always learns perfectly in-distribution, but the question of interest is to understand how well the model  generalizes on new domains depending on the model parametrization and the optimization.

In \cite{gotu}, it is shown that for sparse Boolean functions, i.e., functions defined on $\{+1,-1\}^d$ that depend only on a bounded number of variables, and in such a GOTU setting where training data is available in $\mathcal{U}^c \subseteq \{+1,-1\}^d$ and generalization is tested on $\mathcal{U}$, random feature models learn functions that are interpolators on $\mathcal{U}^c$ with minimal degree-profile: a specific type of polynomials that have minimal degree and also largest mass on the lowest-degree Fourier coefficients. Further, experiments provided in \cite{gotu} show that this `minimal degree bias' also takes place for Transformers. 

In this paper, we study how this picture changes when considering input variables that are not necessarily valued in $\{+1,-1\}^d$. In particular, we consider input variables valued in $\mathcal{X} \subseteq \mathbb{R}^d$, which may be discrete but non-binary or arbitrary real-valued. We study theoretical results for random feature models and provide experiments for Transformers.  

We make this study in two settings: (i) the sparse regime as in \cite{gotu}, where the target function depends effectively on a low input dimension, (ii) the small feature regime, where random features have a weight scale that tends to zero while the input dimension remains fixed. In particular, we explain how these two regimes are related to each other, with (ii) acting as a surrogate of (i) with a regularization term, and investigate both regimes. We further provide experiments for Transformers. 

\section{Paper Contributions and a Motivating Example}
Consider the arithmetic task where inputs are valued in $\mathcal{X}=\{-q/2,\ldots,-1,1,\ldots,q/2\}$,
for some fixed even $q$, and the target function $f: \mathcal{X}^d \to \mathcal{X} \cdot \mathcal{X}$ is given by  
\begin{align*}
f(x_1,x_2,\ldots,x_d)=x_1 \cdot x_2 \,.
\end{align*}
In the GOTU setting, we assume that a set of training examples is given that covers some part of the support and leaves out completely some other part. For instance, consider the case where $x_1$ or $x_2$ is always 1 at training. This is a special case where the model would a priori not have a reason to learn the target function (it does not see effective multiplications). So what will the model learn in that case?

Note that one can model the GOTU constraint in this case as follows:
\begin{align*}
(x_1-1)(x_2-1)=0 \quad \equiv \quad x_1x_2=x_1+x_2-1 \, .
\end{align*}
One possible outcome is that the model could learn a function close to $\widehat{f}(x)=x_1+x_2-1$. This is explained by the following intuition: this function is a correct interpolator of the training data, and it has lowest possible degree. Assuming that such models have a bias towards lower degree polynomials, this function may be learned. This turns out to be a correct intuition in the Boolean case, i.e., when $q=2$. More precisely, this was proved by \citet{gotu} for classic RF models when $d$ diverges, i.e., the `sparse regime', and experiments supporting a similar outcome for Transformers were also obtained. In this paper, as a first contribution, we  show that this outcome also takes place when the target is not sparse (e.g., $d=2$) but the random features have a vanishing variance, which we call the `small feature regime'; this in fact provides a surrogate regime to the sparse regime.  

What happens now if $q>2$? Would such models still learn a function close to $\widehat{f}(x)=x_1+x_2-1$? In this paper, we show the following: 
\begin{enumerate} 
\item (Small-feature RF and any target) Yes, this intuition is still correct for RF models in the small feature regime and any real inputs. See Theorem \ref{theorem: 1} for a formal statement.
\item (Sparse target on arbitrary inputs) No, this intuition is incorrect for classic RF models with general activa- tions on sparse targets, as the model can learned higher degree polynomials, although some activations such as sigmoid appear to still obey the minimal degree rule; similarly, this intuition is not correct in general for Transformers, as we provide experiments with both minimal degree and higher degree interpolators. These experiments are reported in Section \ref{sec:exper}.
\item (RF and sparse target on roots of unities) Yes, this intuition is again correct for classic RF models with general activations (with some regularity condition) and sparse targets if the data is not parametrized as $\mathcal{X}=\{-q/2,\ldots,-1,1,\ldots,q/2\}$ but as $\mathcal{U}=\{e^{2 \pi i k/n} \}_{k=0}^{n-1}$, with the same target $x_1 \cdot x_2$ over the complex numbers (i.e., the target is now the sum of angles of the roots of unities). This is the `natural' extension of the Boolean case (with $q=2$) to larger $q$. Note that this is not due to the fact that the target here becomes the sum of the angles, as the result extends to more generic functions.  We leave it to future work to investigate whether this parametrization could be useful in certain applications; it may also generalize to other groups than roots of unities. We refer to Theorem~\ref{thm:roots-unity} in Section \ref{sec:roots_unity} for the formal statement.      
\end{enumerate}

\section{Background}
\label{sec:setting}

\subsection{Notation}

We denote $\NN$ the set of non-negative integers. If $T \in \NN^d$, we denote $|T| = \Vert T \Vert_1 = T_1 + \dots + T_d$. We define $\NN^d_{\leq p} = \{T \in \NN^d \, \vert \, \vert T \vert \leq p\}$. We also denote $\U_n = \left\{\exp\left(i\frac{2\pi k}{n}\right) \, \vert \, k = 0, \dots, n-1\right\} \subset \CC$ the $n$-roots of unity.  

We denote $\Pi_p(\RR^d)$ the set of multivariate real polynomials on $\RR^d$ with degree less or equal to $p$. Similarly, we denote $\Pi_p(\U_n^d)$ the set of functions on $\U_n^d$ that are the restriction to $\U_n^d$ of a multivariate complex polynomial on $\CC^d$ of degree at most $p$.

If $(V, \Vert . \Vert)$ is a normed vector space, $x\in V$ and $W$ is a subspace of $V$, then $\dist(x, W)$ denotes the distance of $x$ to $W$. We denote $\gamma_d$ the standard Gaussian measure over $\RR^d$.

\subsection{Random Feature Model and Different Regimes}

In the following sections, we will study the random features (RF) model $f_{\textnormal{RF}}(a): \RR^d \rightarrow \RR$, which is defined as
$f_{\textnormal{RF}}(a;x) = \frac{1}{\sqrt{N}} \sum_{i=1}^N a_i \phi_{w_i, b_i}(x)$, $x \in \RR^d$.


Here, $x\in\RR^d$ is the input variable, $a\in\RR^N$ are the trainable parameters, and $\phi_{w_i, b_i}(x), i\in \{1, \dots, N\}$ are the random features, defined by $\phi_{w,b}(x) = \sigma(\innerproduct{w}{x} + b)$, where the weights $w_i$ and biases $b_i$ are sampled randomly and then fixed during the training. 

Traditionally, the weights $w_i$ and biases $b_i$ of random features model are sampled independently and identically distributed (i.i.d.) according to $w_i \sim \mathcal{N}(0, \frac{1}{d} I_d)$, $b\sim \mathcal{N}(0, \frac{1}{d})$.
\cite{gotu} analysed this setting with additional assumptions that the target function $f$ must be $O_d(1)$-sparse (i.e. it must depend on the finite number of variables) while the dimension $d$ diverges.
We call this setting \textit{the sparse regime}.

Additionally, we consider the setting where $w_i \sim \mathcal{N}(0, \varepsilon I_d)$, $b\sim \mathcal{N}(0, \varepsilon)$.
Here, we assume that the dimension $d$ is fixed, but $\varepsilon \rightarrow 0$.
We call this setting \textit{the small features regime}.

As we will see, the sparse and small features settings are related to each other: we can show that they are equivalent up to some regularizer term (see Section \ref{sec: small features motivation}).
However, we will see that, at least for polynomial activation functions, they have different generalization properties in a GOTU setting. 
In small features regime, the random features model converges to a minimum-degree interpolator (MDI) (under some general assumption on polynomial activation function, see Section \ref{sec:main}), while in the sparse regime the convergence to MDI is rather an exception (see Example \ref{example: 1} and Remark \ref{rmk:example1}).
Finally, we note that the sparse regime requires dimension $d$ to be large and target function to be sparse.
On the other hand, the small features regime does not impose any constraint on dimension~$d$ or target function $f$, but requires non-classical initialization of the weights and biases.
Thus, these two setting have different limitations and areas of applicability.

Define the image of $f_{\textnormal{RF}}$ model as the set of functions it can express:
$\im(f_{\textnormal{RF}}) = \{ f_{\textnormal{RF}}(a), \, a \in \RR^N \}$.

\section{Min-Degree Interpolation in Small Features Regime}
\label{sec:main}

Let $\mathcal{U} \subset \RR^d$ be the unseen domain, so that during the training we only see samples from $\mathcal{U}^c = \RR^d \setminus \mathcal{U}$.
We emphasize that being a proper subset of $\RR^d$ is the only constraint we impose on the unseen domain $\mathcal{U}$.
In particular, we can select $\mathcal{U}$ such that the training domain $\mathcal{U}^c$ is finite or countable with a discrete measure defined on it. Similar to \cite{gotu}, we assume that the model has an access to the distribution on the training domain, which makes sampling error zero and allows to state more clear results.

For the activation function $\sigma$ we assume the following.
\begin{assumption} \label{assumption: 3}
    Assume that $\sigma$ is a polynomial of degree~$p$ whose coefficients are non-zero:
    \begin{gather*}
        \sigma(y) = b_p y^p + \ldots + b_1 y + b_0, \quad \text{where } b_p, \dots, b_0 \neq 0 \, .
    \end{gather*}
\end{assumption}

\begin{theorem}\label{theorem: 1}
    Consider training the random features model $f_{\textnormal{RF}}(a;x)$ in the small features regime (with parameter $\varepsilon$) on the polynomial target function $f$. Assume that we observe the target function on the training set $\mathcal{U}^c$, and that the activation function $\sigma$ satisfies Assumption \ref{assumption: 3}.

    For a sufficiently large number $N$ of random features, the model $f_{\textnormal{RF}}$ can interpolate the target function perfectly on $U^c$.

    Among all parameters $a$ such that $f_{\textnormal{RF}}(a)$ interpolates $f$ on $\cU^c$, denote $a^*$ the parameter of minimum $\ell_2$ norm. Denote by $p_*$ the minimum possible degree for a polynomial interpolator of $f$ on the training set $\mathcal{U}^c$.
    Then with high probability, we have\footnote{Since the space $\Pi_d(\RR^d)$ has finite dimension, the convergence in all norms is equivalent in this space.}:
    \begin{gather*}
        \lim_{\varepsilon \to 0} \lim_{N \to \infty} \dist(f_{\textnormal{RF}}(a^*), \Pi_{p_*}) = 0 \, .
    \end{gather*}
\end{theorem}

\begin{remark}
Note that the model will converge to the minimum $\ell^2$ norm solution $a^*$ if trained with (stochastic) gradient descent starting from $a=0$ initialization under the mean squared error loss (in an overparametrized setting).
Thus the theorem describes the bias of gradient descent methods.    
\end{remark}
\begin{remark}
    If the target function $f$ is not a polynomial, we can still apply Theorem \ref{theorem: 1} to describe where the random features model converges.
    Let $\lambda$ be the distribution on the training set $\mathcal{U}^c$ and assume that the mean square error is used to train the model.
    Denote by $\Tilde{f}$ the projection of $f$ in $L^2(\cU^c, \lambda)$ on the space $\Pi_p(\RR^d)$ of polynomials of degree at most $p$:
    \begin{gather*}
        \Tilde{f} = \text{proj}_{\Pi_p(\RR^d)}(f) = \text{argmin}_{h\in\Pi_p(\RR^d)} \|f-h\|_{L^2(\mathcal{U}^c, \lambda)}
    \end{gather*}
    Then the loss function can be decomposed as
    \begin{align*}
        \mathcal{L}(a) 
        & = \mathbb{E}_{x\sim\lambda} \left[(f_{\textnormal{RF}}(a;x) - f(x))^2\right] \\
        & = \|f_{\textnormal{RF}}(a) - f\|_{L^2(\mathcal{U}^c, \lambda)}^2 \\
        & = \|f_{\textnormal{RF}}(a) - \Tilde{f} + \Tilde{f} - f\|_{L^2(\mathcal{U}^c, \lambda)}^2 \\
        & = \|f_{\textnormal{RF}}(a) - \Tilde{f}\|_{L^2(\mathcal{U}^c, \lambda)}^2 + \|\Tilde{f} - f\|_{L^2(\mathcal{U}^c, \lambda)}^2 \, .
    \end{align*}
    The second term in the last expression is independent of~$a$.
    Thus, the training trajectory would be the same as if we trained the model on the target function $\Tilde{f} \in \Pi_p(\RR^d)$, and we may predict where the random features model converges by applying Theorem \ref{theorem: 1} to the target function $\Tilde{f}$.
\end{remark}
\begin{remark}
    Our theorem is not specific to Gaussian parameters of the random features. The proof also works for any weights and biases of the form $w_i = \varepsilon^{1/2} \Bar{w}_i$ and $b_i = \varepsilon^{1/2} \Bar{b}_i$, with $\Bar{w}_i \sim \mu$ and $\Bar{b}_i \sim \nu$ and $\mu,\nu$ are any distributions with all moments finite. 
\end{remark}

We refer to Appendix \ref{sec:lemmas} for the proof of the theorem. The proof can be decomposed into two parts. First, we show that for fixed $\varepsilon$, as $N\rightarrow \infty$ the random features model, denoted as $g$, converges to the minimizer of the quadratic form $\hat{g}^T \Phi^{-1} \hat{g}$, denoted as $g_{\varepsilon}$, where $\Phi$ is a feature kernel matrix, and $\hat{g}$ is the vector of Hermite coefficients of $g$. The proof of this part follows the scheme of Theorem 3.8 from \cite{gotu}. Second, we analyze how this minimizer $g_{\varepsilon}$ behaves in the limit of $\varepsilon\rightarrow0$ and prove that $\dist(g_{\varepsilon}, \Pi_{p_*}) \rightarrow 0$. This part of our proof is original. In the Boolean case of \cite{gotu}, the matrix $\Phi$ was diagonal. Hence it was enough to estimate its diagonal entries, which directly leads to the approximation of its inverse. However, in general case matrix $\Phi$ is non-diagonal. Thus, we estimate all entries of the matrix $\Phi$ and derive the suitable upper and lower bounds on the quadratic form $\hat{g}^T \Phi^{-1} \hat{g}$ (Lemma \ref{lemma: 5} and Corollary \ref{corollary: 2} in Appendix \ref{sec:lemmas}). This is the part where most of the technical difficulty and conceptual novelty lies. This is also where the big picture changes with the min-degree bias of \cite{gotu} breaking, if we do not use the small features regime (see Example \ref{example: 1} for the demonstration of min-degree bias breaking).

\section{Motivations for the Small Features Setting}
\label{sec: small features motivation}

In this section, we derive equivalences between the setting with fixed dimension and small features, and the setting with diverging dimension and $O(1)$ features. 

Let $k$ denote a fixed dimension and $d \gg 1$ denote a large dimension. We set ourselves in the multi-index model where we seek to learn a function $f: \RR^d \to \RR$ of the form 
\begin{equation*}
    f(x) = \varphi(U^\top x) \, ,
\end{equation*}
where $U \in \RR^{d \times k}$, $U^\top U = I_k$ and $\varphi : \RR^k \to \RR$. 

We consider the approximation of $f$ with random features:
\begin{equation*}
    {f}_{\textnormal{RF}}(a;x) = \sum_{i=1}^N a_i \sigma(\langle w_i, x \rangle + c_i) \, , 
\end{equation*}
where $w_i \sim \cN(0, \frac{1}{d} I_d)$ and $c_i \in \RR$. (The reasoning actually works for other random features, this is simply to set an order of magnitude for the $w_i$.) 

We define the loss function in the approximation 
\begin{equation*}
    \mathcal{L}(a) = \frac{1}{2} \E_x \left[\left(f(x) - {f}_{\textnormal{RF}}(a;x)\right)^2\right] \, .
\end{equation*}
Denote $P_{\parallel}$ the orthogonal projection onto $\im(U)$ and $P_\perp$ the orthogonal projection onto $\ker U^\top = (\ker U)^\perp$. If $q \in \RR^d$, we denote $q_\parallel = P_\parallel q$ and $q_\perp = P_\perp q$. We make the following assumption on the input distribution of $x$. 

\begin{assumption}
\label{ass:decomposition}
    We assume that $x_\parallel$ and $x_\perp$ are independent.
\end{assumption}

\begin{example}
This assumption holds in many cases of interest. We show two examples.
\begin{enumerate}
    \item If $x \sim \cN(0, I_d)$, then the assumptions holds. Indeed, then $x_\parallel \sim \cN(0, P_\parallel)$ and $x_\perp \sim \cN(0, P_\perp)$ are independent. 
    \item If $x \sim \Unif(\{-1,1\}^d)$ and the columns of $U$ are a subset of the canonical basis. (This means that the multi-index model is sparse, meaning that it only depends on a subset of the coordinates.) In this case, $x_\parallel$ and $x_\perp$ are independent with uniform distribution on hypercubes of respective dimension $k$ and $d-k$.
\end{enumerate}
\end{example}

Under Assumption \ref{ass:decomposition}, we compute
\begin{align}
    \mathcal{L}(a) &= \frac{1}{2} \E_x \left[\left(f(x) - {f}_{\textnormal{RF}}(a;x)\right)^2\right] \label{eq:risk} \\
    &= \frac{1}{2} \E_{x_\parallel}\E_{x_\perp}  \left[\left(f(x) - {f}_{\textnormal{RF}}(a;x)\right)^2\right] \, . \label{eq:risk-bis}
\end{align}
As $U^\top x$ is independent of $x_\perp$, we have
\begin{align*}
    &\E_{x_\perp}  \left[\left(f(x) - {f}_{\textnormal{RF}}(a;x)\right)^2\right] \\
    &= \E_{x_\perp}  \left[\left(\varphi(U^\top x) - {f}_{\textnormal{RF}}(a;x) \right)^2\right] \\
    &= \left(\varphi(U^\top x) - \E_{x_\perp} {f}_{\textnormal{RF}}(a;x) \right)^2 + {\var}_{x_\perp}  {f}_{\textnormal{RF}}(a;x) = \\
    &\left(\varphi(U^\top x) - \sum_{i=1}^N a_i \E_{x_\perp} \sigma(\langle w_{i\parallel}, x_\parallel \rangle + \langle w_{i\perp}, x_\perp \rangle + c_i) \right)^2 \\
    &+ {\var}_{x_\perp} \left( \sum_{i=1}^N a_i \sigma(\langle w_i, x \rangle + c_i)\right) \, .
\end{align*}
We denote $\overline{\sigma}_i(\lambda) = \E_{x_\perp}\left[\sigma(\lambda + \langle w_{i\perp}, x_\perp \rangle)\right]$. This corresponds to a smoothed version of the non-linearity~$\sigma$. (For instance, in the case of Gaussian inputs $x \sim \cN(0,I_d)$, the smoothing noise $\langle w_{i\perp}, x_\perp \rangle$ would be Gaussian $\cN(0, \Vert w_{i\perp} \Vert^2)$). We then obtain:
\begin{align*}
    &\E_{x_\perp}  \left[\left(f(x) - {f}_{\textnormal{RF}}(a;x)\right)^2\right] \\
    &\qquad= \left(\varphi(U^\top x) - \sum_{i=1}^N a_i  \overline{\sigma}_i(\langle w_{i\parallel}, x_\parallel \rangle + c_i) \right)^2 \\
    &\qquad+ \sum_{i,j=1}^N a_i a_j{\cov}_{x_\perp} \left( \sigma(\langle w_i, x \rangle + c_i), \sigma(\langle w_j, x \rangle + c_j)\right) \, .
\end{align*}
Thus, returning to \eqref{eq:risk}--\eqref{eq:risk-bis}, we obtain 
\begin{align}
    \mathcal{L}(a) &= \frac{1}{2} \E_{x_\parallel} \left[\left(\varphi(U^\top x) - \sum_{i=1}^N a_i  \overline{\sigma}_i(\langle w_{i\parallel}, x_\parallel \rangle + c_i) \right)^2\right] \nonumber\\ 
    &\qquad+ \frac{1}{2} a^\top \Lambda a \nonumber \\
    &= \frac{1}{2} \E_{z} \left[\left(\varphi(z) - \sum_{i=1}^N a_i  \overline{\sigma}_i(\langle U^\top w_{i}, z \rangle + c_i) \right)^2\right]\label{eq:equivalence} \\
    &\qquad+ \frac{1}{2} a^\top \Lambda a \, , \label{eq:equivalence-bis}
\end{align}
where $z = U^\top x$ and 
\begin{equation*}
    \Lambda_{i,j} = \E_{x_\parallel} \left[{\cov}_{x_\perp} \left( \sigma(\langle w_i, x \rangle + c_i), \sigma(\langle w_j, x \rangle + c_j)\right)\right] \, .
\end{equation*}
The take-home message is that the high-dimensional regression problem in $x \in \RR^d$ reduces to a lower dimensional regression problem in $z \in \RR^k$ with an additional regularization term $a^\top \Lambda a$ and modified features. The non-linearities are smoothed and the feature vectors $w_i$ are projected onto~$U$. If $w_i \sim \cN(0,\frac{1}{d}I_d)$, then $U^\top w_i \sim \cN(0, \frac{1}{d} I_k)$. This gives small features: $\E \Vert U^\top w_i \Vert^2 = \frac{k}{d}$. 

As a consequence, minimizing only the first term in \eqref{eq:equivalence}--\eqref{eq:equivalence-bis}, and taking the minimum norm solution, would lead to a minimum degree solution by Section \ref{sec:main}. However, the second term, that controls the variance of the model in the orthogonal direction, actually has an important effect on the chosen minimizer. As we demonstrate in Example~\ref{example: 1} below, in some cases it can break down the MDI bias.

\begin{example} \label{example: 1}
    Consider the target function be $f(x) = 1$ with GOTU constraint $x_1 = 1$, and assume that the support of the training distribution contains a subset of the hyperplane $\{x \in \RR^d \, \vert x_1 = 1\}$ of the form $\{1\}\times S_2\times \dots \times S_d$ with $|S_2|, \dots,|S_d| \geq 3$. Then the MDI is given by the target function itself, but the random features model trained in sparse regime with $\sigma(x) = (1+x)^2$ converges to $f_{\textnormal{RF}}(x) = \frac{2}{5} x_1 + \frac{3}{5}$ (as $N \rightarrow \infty$ before $d \rightarrow \infty$).
    This shows that the random features model in general does not converge to the MDI in the sparse case, provided that the training distribution has strictly more than two inputs on each coordinate. Thus it is not possible to naively extend the results of \cite{gotu} beyond the hypercube $\{-1,1\}^d$. 
\end{example}
\label{rmk:example1}
See the proof of Example \ref{example: 1} in Appendix \ref{sec:example proof}, and the simulation results in Figure \ref{fig: 6}. From Figure \ref{fig: 6} we see that even for moderate values $d=15$ and $N=1024$, the model converges close to the asymptotic value.
\begin{remark}
    Empirically we observed the lost of MDI property in this example for all polynomial activation functions that we checked, e.g. $(1+x)^2, (1+x)^2-1, x^2+x, (1+x)^3, (1+x)^4$. Thus, we believe that it is a general property for polynomial activations rather than a degenerate case.
\end{remark}

\section{MDI for Data Embedded in Roots of Unity}
\label{sec:roots_unity}

We recall that $\U_n = \left\{\exp\left(i\frac{2\pi k}{n}\right) \, \vert \, k = 0, \dots, n-1\right\} \subset \CC$ denotes the $n$-roots of unity. Consider learning a target function $f:\U_n^d \to \CC$ using a random feature model 
\begin{equation*}
    f_{\textnormal{RF}}(a;x) = \frac{1}{\sqrt{N}} \sum_{i=1}^N a_i \phi_{w,b}(x) \, ,
\end{equation*}
where the random features are defined as $\phi_{w,b}(x) = \sigma(\langle w, x \rangle + b)$. Compared to Section \ref{sec:setting}, this section takes the suitable generalization to the complex case: $a_i \in \CC$, $b_i \in \CC$ with distribution $(\Re b_i, \Im b_i) \sim \cN\left(0, \frac{1}{d}I_2\right)$, $w_i \in \CC^d$ with distribution $(\Re w_{i1}, \Im w_{i1}, \dots, \Re w_{id}, \Im w_{id}) \sim \cN\left(0, \frac{1}{d}I_{2d}\right)$, $\sigma: \CC \to \CC$ and $\langle w_i, x \rangle = \overline{w}_{i1}x_1 + \dots + \overline{w}_{id}x_d$. 

Let $\cU \subset \U_n^d$ denote the subset of which $f$ is unseen and denote 
\begin{align*}
    a^* = \argmin_{a: f_{\textnormal{RF}}(a;x) = f(x), \, x\in \cU^c} \Vert a \Vert 
\end{align*}
the minimum norm interpolant of $f$ on the training domain. We recall that $\Pi_p(\U_n^d)$ denotes the set of complex polynomial functions of degree $p$ on $\U_n^d$ (i.e.~the set of functions on $\U_n^d$ that are the restriction to $\U_n^d$ of a multivariate complex polynomial on $\CC^d$ of degree at most $p$). 

\begin{theorem}
\label{thm:roots-unity}
    Denote by $p_*$ the minimum possible degree for a polynomial interpolator of $f$ on the set $\cU^c$. Then 
    \begin{equation*}
        \lim_{d \to \infty} \lim_{N \to \infty} d(f_{\textnormal{RF}}(a^*,.), \Pi_p(\U_n^d)) = 0 \, .
    \end{equation*}
\end{theorem}

This result is proved in Appendix \ref{ap:roots-unity}.


\section{Experiments}
\label{sec:exper}

\subsection{Experiments Setup}

We run experiments\footnote{Code: https://github.com/DenisPushkin/GOTU-real-valued} with the random features (RF) model and Transformer \cite{vaswani2017attention}.
For the RF model, we sample 65536 training points from the standard Gaussian distribution (except for the coordinates affected by the GOTU constraint, for which we simply hard code the required value) and train the model using Gradient Descent with line search (we refer to Appendix \ref{sec: gradient line search} for the exact procedure).
For convex functions with Lipschitz continuous gradient, this method provably converges to the global optimum and does not require the learning rate tuning.
As for Transformer, we use AdamW optimizer \cite{loshchilov2017decoupled} (without weight decay), and for each batch we generate 256 random samples satisfying the GOTU constraint on the fly, imitating the access to the whole data on the seen domain.
For the exact Transformer architecture we used, see Appendix \ref{sec: transformer}.

In both cases, we train the model on the data satisfying GOTU constraints and then evaluate on the full domain to capture its behavior on the unseen data.
In case of the real-valued training domain, we evaluate the Hermite coefficients of the model. 
Note that the choice of Hermite polynomial basis is arbitrary, yet sufficient for our needs.
Indeed, we are mainly interested in the polynomial degree of the function learnt by the model, which does not depend on the choice of polynomial basis. 

When the training domain is a discrete grid, i.e. represented by $x \in \mathcal{X}^d$, where $\mathcal{X}$ is a finite set, we evaluate the model's coefficient considering it as a simple multivariate polynomial.
It is justified by the fact that the set $\mathcal{B} = \{\prod_{i=1}^d x_i^{t_i} \: \vert \: 0 \leq t_i \leq |\mathcal{X}| - 1 \: \forall i \}$ of monomials with degree at most $|\mathcal{X}| - 1$ in each variable forms a basis of functions in $\mathcal{X}^d$.
This result may be derived as a consequence of Combinatorial Nullstellensatz \cite{alon1999combinatorial}.
Note that in a special case where $\mathcal{X} = \{\pm 1\}$, this basis produces the Fourier-Walsh basis of boolean functions, which was a central ingredient of MDI analysis in \cite{gotu}.

\subsection{Small Features Regime}
First, we empirically confirm convergence to min-degree interpolator (MDI) for RF model in small features regime with polynomial activation (Theorem \ref{theorem: 1}).
We run two experiments: 1) $f(x) = 1$ with GOTU constraint $x_1 = 1$ (see Figure \ref{fig: 1}) and 2) $f(x) = x_2^2 + x_2 + 1$ with $x_1 = 1$ (Figure \ref{fig: 2}).
In both cases, the target function $f$ is itself an MDI, but in the second case the MDI is not unique: any function of the form $f(x) + (x_1-1) \Delta(x)$ with $deg(\Delta) \leq 1$ would be an MDI.
As predicted by Theorem \ref{theorem: 1}, the RF model converges to MDI in both cases.
However, in the second experiment, the trained model depends on the variable $x_1$, while the target function does not.
This shows that the random features model in small features regime does not always converge to "the simplest"\footnote{One possible formalization of "the simplest" MDI is a minimum degree-profile interpolator, defined in \cite{gotu}.} MDI.

\begin{figure}
\vskip 0.2in
\begin{center}
\centerline{\includegraphics[width=\columnwidth]{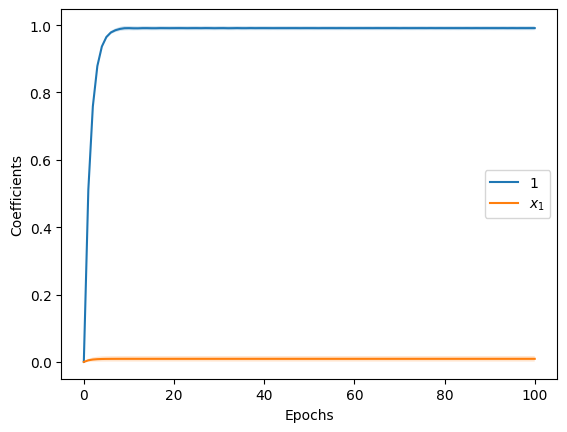}}
\caption{Training the random features model on $f(x) = 1$ with GOTU constraint $x_1 = 1$ in small features regime. Here, $d=2$, $N=256$, $\varepsilon=(0.05)^2$, and $\sigma(x) = (1+x)^2$.}
\label{fig: 1}
\end{center}
\vskip -0.2in
\end{figure}

\begin{figure}
\vskip 0.2in
\begin{center}
\centerline{\includegraphics[width=\columnwidth]{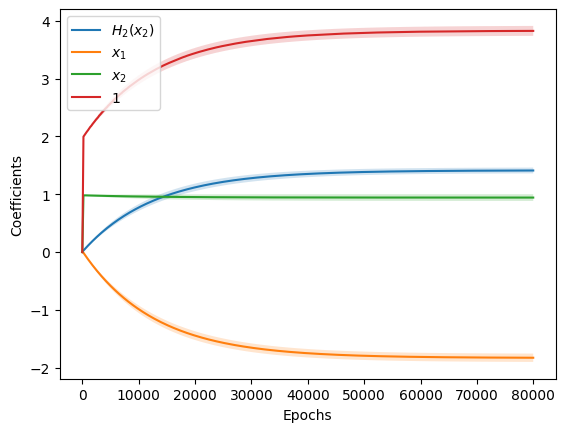}}
\caption{Training the random features model on $f(x) = x_2^2 + x_2 + 1$ with GOTU constraint $x_1 = 1$ in small features regime. Here, $d=2$, $N=16384$, $\varepsilon = (0.05)^2$ and $\sigma(x) = (1+x)^2$. The model converged to MDI, but not "the simplest one", since it depends on $x_1$.}
\label{fig: 2}
\end{center}
\vskip -0.2in
\end{figure}
    
\begin{table*}
\caption{Training the random features model on $f(x) = x_1x_2$, $x\in\RR^d$ with GOTU constraint $(x_1-1)(x_2-1)=0$ in sparse regime. Here, $d=15$ and $N=40000$. The MDI is given by $x_1+x_2-1$.}
\label{table: 10}
\vskip 0.15in
\begin{center}
\begin{small}
\begin{sc}
\begin{tabular}{l|c|c|c|c}
\toprule
Activation & 1 & $x_1$ & $x_2$ & $x_1 x_2$ \\
\midrule
ReLU & $-0.952 \pm 0.002$ & $0.955 \pm 0.005$ & $0.955 \pm 0.004$ & $0.042 \pm 0.009$ \\
Shifted ReLU & $-0.957 \pm 0.002$ & $0.958 \pm 0.004$ & $0.958 \pm 0.005$ & $0.039 \pm 0.009$ \\
Sigmoid & $-1.013 \pm 0.003$ & $0.996 \pm 0.003$ & $0.996 \pm 0.006$ & $-0.001 \pm 0.009$ \\
Softplus & $-0.975 \pm 0.003$ & $0.978 \pm 0.004$ & $0.977 \pm 0.006$ & $0.022 \pm 0.010$ \\
\bottomrule
\end{tabular}
\end{sc}
\end{small}
\end{center}
\vskip -0.1in
\end{table*}

\subsection{Random Features Model with Standard Activations}

Now, we examine the RF model with standard (non-polynomial) activations.
First, we compare the sparse and the small features regimes on the target $f(x) = 1$ with GOTU constraint $x_1 = 1$.
For both regimes, we use the same model architecture with $d=15$ input dimension and $N = 1024$ random features and compare the same set of activation functions.
You can see the result in Table \ref{tab: 1} for small features regime and Table \ref{tab: 2} for sparse regime.
We see that in sparse regime, the RF model in general learns a linear interpolator instead of the constant one (the only exception is Sigmoid activation). 
For small features regime, the RF model converges to the constant interpolator for all activations, except for ReLU.
We conjecture that this happens because $\text{ReLU}(0) = 0$, which breaks the Assumption \ref{assumption: 3}, used in the Theorem \ref{theorem: 1} (note that this assumption was stated only for polynomial activations).
In contrast, with Shifted ReLU activation, given by $\text{Shifted ReLU}(x) = \text{ReLU}(x) - 1$, the convergence to MDI is restored.
Hence, we conjecture that Assumption \ref{assumption: 3} is also a prerequisite for non-polynomial activations to guarantee the convergence to MDI in small features regime.

Next, we train RF model in sparse regime on $f(x) = x_1x_2$ with GOTU constraint $(x_1-1)(x_2-1)=0$ (see Table \ref{table: 10}), where the MDI is given by $x_1+x_2-1$.
In this example, the RF model converges close to MDI for all activations we tried, which shows that, for some target functions, the MDI bias may also holds for the RF model in sparse regime.

\begin{table}
\caption{Training the random features model on $f(x) = 1$, $x\in\RR^d$ with GOTU constraint $x_1 = 1$ in small features regime with $\varepsilon = (0.03)^2$. Here, $d=15$ and $N=1024$.}
\label{tab: 1}
\vskip 0.15in
\begin{center}
\begin{small}
\begin{sc}
\begin{tabular}{l|c|c}
\toprule
Activation & 1 & $x_1$ \\
\midrule
 $(1+x)^2$ & $0.997 \pm 0.002$ & $0.001 \pm 0.003$ \\
ReLU & $0.564 \pm 0.009$ & $0.430 \pm 0.010$ \\
Shifted ReLU & $1.000 \pm 0.000$ & $-0.001 \pm 0.003$ \\
Sigmoid & $1.000 \pm 0.000$ & $-0.001 \pm 0.003$ \\
Softplus & $1.000 \pm 0.001$ & $-0.001 \pm 0.003$ \\
\bottomrule
\end{tabular}
\end{sc}
\end{small}
\end{center}
\vskip -0.1in
\end{table}

\begin{table}
\caption{Training the random features model on $f(x) = 1$, $x\in\RR^d$ with GOTU constraint $x_1 = 1$ in sparse regime. Here, $d=15$ and $N=1024$.}
\label{tab: 2}
\vskip 0.15in
\begin{center}
\begin{small}
\begin{sc}
\begin{tabular}{l|c|c}
\toprule
Activation & 1 & $x_1$ \\
\midrule
 $(1+x)^2$ & $0.624 \pm 0.017$ & $0.374 \pm 0.017$ \\
ReLU & $0.564 \pm 0.009$ & $0.431 \pm 0.011$ \\
Shifted ReLU & $0.782 \pm 0.009$ & $0.217 \pm 0.011$ \\
Sigmoid & $0.992 \pm 0.003$ & $0.007 \pm 0.002$ \\
Softplus & $0.789 \pm 0.010$ & $0.208 \pm 0.012$ \\
\bottomrule
\end{tabular}
\end{sc}
\end{small}
\end{center}
\vskip -0.1in
\end{table}

Finally, Example \ref{example: 1} illustrates that the RF model in sparse regime with polynomial activation generally does not learn the MDI.
But how far can it depart from the MDI, e.g. can the degree of the trained model exceed the minimum degree of interpolator by more than one?
In Figure \ref{fig: 5} we demonstrate that the RF model with $\sigma(x) = (1+x)^4$ activation trained on $f(x) = 1$ with GOTU constraint $x_1 = 1$ learns the quadratic function.
It shows that the RF model in sparse regime can exceed the MDI by more than one degree.
We conjecture that this gap can be arbitrarily large as we increase the degree of the polynomial activation function.

\begin{figure}
\vskip 0.2in
\begin{center}
\centerline{\includegraphics[width=\columnwidth]{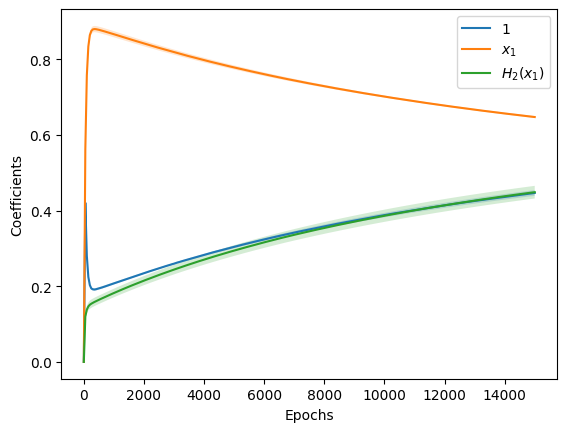}}
\caption{Training the random features model on $f(x) = 1$ with GOTU constraint $x_1 = 1$ in sparse regime with $\sigma(x) = (1+x)^4$ activation. Here, $d=15$, $N=3\cdot10^5$, and $H_2(x_1)$ denotes the normalized second degree Hermite polynomial. The MDI is a constant function $1$, but the model learns the quadratic function.}
\label{fig: 5}
\end{center}
\vskip -0.2in
\end{figure}

\subsection{Transformer and Random Features with Discrete Input}

Finally, we consider the input variable $x$ from the discrete space.
It allows us to apply Transformer model in our experiments.

For Transformer, we run the following experiments: 1) $f(x) = 1$,  $x \in \{-2, -1, 0, 1, 2\}^d$ with GOTU constraint $x_1=1$ in dimension $d=15$ (see Figure \ref{fig: 4}) and 
2) $f(x) = x_1x_2$, $x\in \{ -1, 0, 1\}^d$ with $(x_1-1)(x_2-1)=0$ and $d=15$ (Figure \ref{fig: 3}).
In the first experiment, the MDI is given by a target function $f(x) = 1$ itself, and Transformer indeed learns a constant function and neglects the constrained variable $x_1$.
In the second experiment, the MDI is given by a linear function $f(x) = x_1 + x_2 - 1$.
In this case, the Transformer's behavior depends on the learning rate. 
With a moderate learning rate of $10^{-4}$ (Figure \ref{fig: 3}, left), Transformer's coefficients are noisy at the first half of the training, but then sharply stabilize and converge to the interpolator\footnote{Of course, it's just a hypothesis that Transformer converges to this exact function. In the experiments, the final coefficients of the Transformer are very close to $\pm 1/2$, but never equal to it.} $f_{int}(x) = \frac{1}{2}(x_1 + x_2 - x_1^2 + x_1 x_2 - x_2^2 + x_1^2 x_2^2)$ (see the exact coefficients in Table \ref{tab: 3}). This shows that Transformer consistently learns the interpolator of degree 4 instead of the linear MDI.
We also repeat this experiment with $10^{-5}$ learning rate - the same one which leads to the min-degree interpolator for boolean functions in \cite{gotu} (Figure \ref{fig: 3}, right).
In this case, Transformer's coefficients did not converge even after $6\cdot10^{6}$ iterations. However, the trajectory suggests that the coefficient of $x_1 x_2$ is non-negligible, which means that Transformer learns at least a quadratic function.
Moreover, the coefficient $x_1 x_2$ is likely to dominate all other coefficients, implying that the learnt function is not an MDI even in a leaky sense (i.e. the high-degree monomials are not dominated by the low-degree alternatives).

Note the crucial difference with the boolean case, where Transformer converges to MDI when trained on the same target function with $10^{-5}$ learning rate \cite{gotu}.\footnote{The other distinction between our experiments and the ones made by \cite{gotu} is that the latter stops the training when the loss reaches a low enough threshold, while we train the model longer until its coefficients are well stabilized. It may happen that (leaky) MDI bias is stronger when lower learning rates or early stopping is used; we leave this hypothesis, as well as the evolution of the MDI on long training past a `low' threshold for future research.} 
This shows that min-degree bias for Transformer does not generalize beyond the boolean domain.

We also train RF model on $f(x) = 1$, $x \in \{-2, -1, 0, 1, 2\}^d$ with GOTU constraint $(x_1-1)(x_2-1)=0$ in dimension $d = 15$, using $\sigma(x) = (1+x)^2$ activation (see Figure \ref{fig: 6}). We observe that the RF model learns a linear function, while the MDI is given by a constant. It confirms the statement of Example \ref{example: 1} that even for discrete domains, the RF model in sparse regime with polynomial activation does not converge to MDI.

\begin{figure}
\vskip 0.2in
\begin{center}
\centerline{\includegraphics[width=\columnwidth]{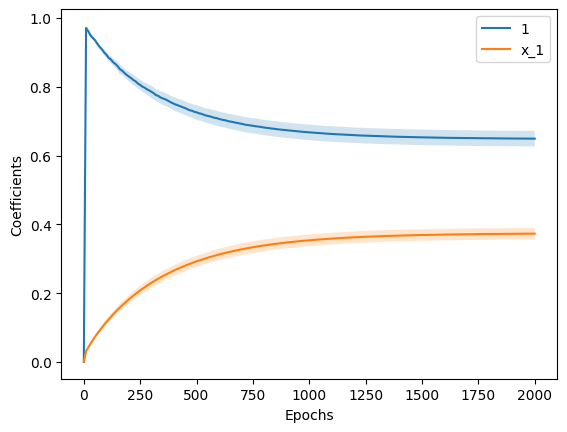}}
\caption{Training the random features model on $f(x) = 1$, $x\in\{-2, -1, 0, 1, 2\}^d$ with GOTU constraint $x_1=1$ and $\sigma(x) = (1+x)^2$ activation. Here, $d=15$, $N = 1024$. While MDI is given by a constant function, the model learns a linear interpolator.}
\label{fig: 6}
\end{center}
\vskip -0.2in
\end{figure}

\begin{figure}
\vskip 0.2in
\begin{center}
\centerline{\includegraphics[width=\columnwidth]{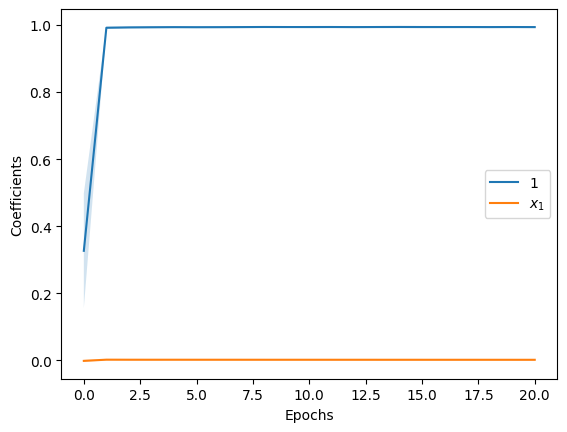}}
\caption{Training Transformer on $f(x) = 1$, $x\in\{-2, -1, 0, 1, 2\}^d$ with GOTU constraint $x_1=1$ using AdamW optimizer with $10^{-4}$ learning rate. Here, $d=15$.}
\label{fig: 4}
\end{center}
\vskip -0.2in
\end{figure}

\begin{figure*}
     \centering
     \begin{subfigure}[b]{0.45\textwidth}
         \centering
         \includegraphics[width=\textwidth]{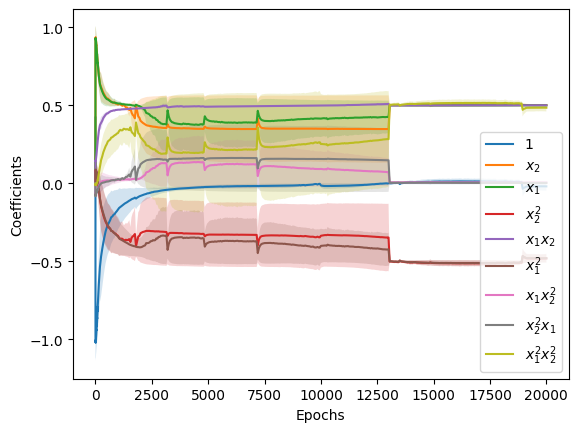}
     \end{subfigure}
     \begin{subfigure}[b]{0.45\textwidth}
         \centering
         \includegraphics[width=\textwidth]{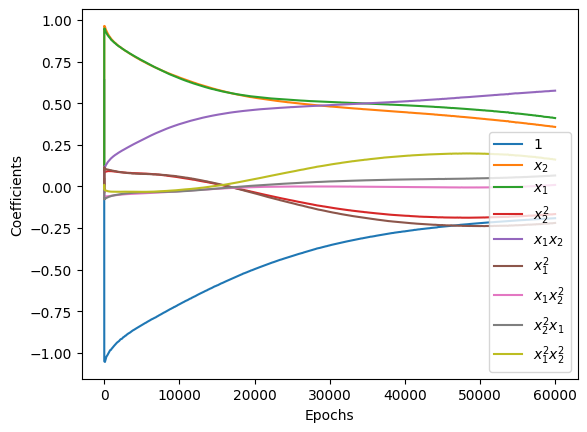}
     \end{subfigure}
     \caption{Training Transformer on $f(x) = x_1x_2$, $x\in\{-1, 0, 1\}^d$ with GOTU constraint $(x_1-1)(x_2-1)=0$ in dimension $d=15$ using AdamW optimizer. The MDI is given by $x_1+x_2-1$. We used the learning rate $10^{-4}$ for the left plot, and $10^{-5}$ for the right one.}
     \label{fig: 3}
\end{figure*}

\begin{table}
\caption{Final coefficients of the trained Transformer on $f(x) = x_1x_2$, $x\in \{ -1, 0, 1\}^d$ with GOTU constraint $(x_1-1)(x_2-1)=0$ in dimension $d=15$. Here, we used AdamW optimizer with $10^{-4}$ learning rate.}
\label{tab: 3}
\vskip 0.15in
\begin{center}
\begin{small}
\begin{sc}
\begin{tabular}{c|r}
\toprule
Monomial & Coefficient \\
\midrule
1 & $-0.020 \pm 0.028$ \\
$x_1$ & $0.499 \pm 0.000$ \\
$x_2$ & $0.501 \pm 0.001$ \\
$x_1^2$ & $-0.480 \pm 0.027$ \\
$x_1 x_2$ & $0.500 \pm 0.004$ \\
$x_2^2$ & $-0.481 \pm 0.027$ \\
$x_1^2 x_2$ & $0.002 \pm 0.005$ \\
$x_2^2 x_1$ & $0.005 \pm 0.003$ \\
$x_1^2 x_2^2$ & $0.484 \pm 0.024$ \\
\bottomrule
\end{tabular}
\end{sc}
\end{small}
\end{center}
\vskip -0.1in
\end{table}

\section{Additional Related Literature}

This paper is a generalization and extension of \cite{gotu}. 

Out-of-distribution generalization is a critical aspect of machine learning that has been studied both in theory \cite{ben2006analysis, mansour2009domain, redko2020survey} and in practice \cite{gulrajani2020search, miller2021accuracy, wiles2022a}. Our work considers an extreme case of distribution shift in which part of the domain is entirely unseen during the training.
OOD generalization and the ability to extrapolate have also been used as proxies for measuring the reasoning capabilities of neural networks \cite{saxton2019analysing, Zhang2021PointerVR, csordas2021devil, zhang2022unveiling} as these models are prone to memorization of training samples \cite{carlini2019secret-mem1,feldman2020neural-mem2,kandpal2022deduplicating-mem3,carlini2022quantifying-mem4, Zhang2021PointerVR} or learning undesirable shortcuts \cite{zhang2022unveiling}.
A special case is length generalization \cite{zaremba2014learning, lake2018generalization, hupkes2020compositionality, zhang2022unveiling, anil2022exploring-length}, i.e., generalization to the input lengths beyond what is seen during the training.

It has been shown that training with gradient descent imposes particular implicit regularization on the solutions found by the models such as sparsity \cite{moroshko2020implicit}, norm minimization \cite{bartlett2021deep}, and margin maximization (in linear classification setting) \cite{soudry2017implicit}. 
This implicit regularization (or implicit bias) of neural networks trained with gradient-based algorithms has been used to explain the generalization of (often overparametrized) models \cite{bartlett2021deep}.
These results depend on the optimizer \cite{gunasekar2018characterizing} and model \cite{gunasekar2018implicit} and are usually proven for simple models such as linear models \cite{soudry2017implicit,  yun2020unifying, jacot2021saddle} including diagonal linear neural networks \cite{gunasekar2018implicit, moroshko2020implicit}. Our result for the random feature model builds upon the implicit bias toward solutions with minimum norm \cite{bartlett2021deep}. Related to us is also the spectral bias \cite{xu2019frequency, rahaman2019spectral} stating that neural networks, when learning a function in continuous settings, capture the lower frequency components faster (note that degree in Boolean functions plays a similar role to the frequency). 

\section{Conclusion}
This paper shows that the min-degree bias in the non-Boolean case is mitigated by various phenomena. One setting admits a clear min-degree bias for the considered models, that with tokens being roots of unities. Moreover, Transformer may still admit some leaky min-degree bias, and it remains open to understand what else drives the bias of Transformers (e.g., the influence of close samples).

\section*{Impact Statement}

This paper presents work whose goal is to advance the field of Machine Learning. There are many potential societal consequences of our work, none which we feel must be specifically highlighted here.

\bibliography{references.bib}

\begin{thebibliography}{38}
\providecommand{\natexlab}[1]{#1}
\providecommand{\url}[1]{\texttt{#1}}
\expandafter\ifx\csname urlstyle\endcsname\relax
  \providecommand{\doi}[1]{doi: #1}\else
  \providecommand{\doi}{doi: \begingroup \urlstyle{rm}\Url}\fi

\bibitem[Abbe et~al.(2023)Abbe, Bengio, Lotfi, and Rizk]{gotu}
Abbe, E., Bengio, S., Lotfi, A., and Rizk, K.
\newblock Generalization on the unseen, logic reasoning and degree curriculum.
\newblock In \emph{ICML}, 2023.
\newblock URL \url{https://arxiv.org/abs/2301.13105}.

\bibitem[Alon(1999)]{alon1999combinatorial}
Alon, N.
\newblock Combinatorial nullstellensatz.
\newblock \emph{Combinatorics, Probability and Computing}, 8\penalty0 (1-2):\penalty0 7--29, 1999.

\bibitem[Anil et~al.(2022)Anil, Wu, Andreassen, Lewkowycz, Misra, Ramasesh, Slone, Gur-Ari, Dyer, and Neyshabur]{anil2022exploring-length}
Anil, C., Wu, Y., Andreassen, A., Lewkowycz, A., Misra, V., Ramasesh, V., Slone, A., Gur-Ari, G., Dyer, E., and Neyshabur, B.
\newblock Exploring length generalization in large language models.
\newblock \emph{arXiv preprint arXiv:2207.04901}, 2022.

\bibitem[Bakhtin et~al.(2019)Bakhtin, van~der Maaten, Johnson, Gustafson, and Girshick]{bakhtin2019phyre}
Bakhtin, A., van~der Maaten, L., Johnson, J., Gustafson, L., and Girshick, R.
\newblock Phyre: A new benchmark for physical reasoning.
\newblock \emph{Advances in Neural Information Processing Systems}, 32, 2019.

\bibitem[Bartlett et~al.(2021)Bartlett, Montanari, and Rakhlin]{bartlett2021deep}
Bartlett, P.~L., Montanari, A., and Rakhlin, A.
\newblock Deep learning: a statistical viewpoint.
\newblock \emph{Acta numerica}, 30:\penalty0 87--201, 2021.

\bibitem[Ben-David et~al.(2006)Ben-David, Blitzer, Crammer, and Pereira]{ben2006analysis}
Ben-David, S., Blitzer, J., Crammer, K., and Pereira, F.
\newblock Analysis of representations for domain adaptation.
\newblock \emph{Advances in neural information processing systems}, 19, 2006.

\bibitem[Carlini et~al.(2019)Carlini, Liu, Erlingsson, Kos, and Song]{carlini2019secret-mem1}
Carlini, N., Liu, C., Erlingsson, {\'U}., Kos, J., and Song, D.
\newblock The secret sharer: Evaluating and testing unintended memorization in neural networks.
\newblock In \emph{USENIX Security Symposium}, volume 267, 2019.

\bibitem[Carlini et~al.(2022)Carlini, Ippolito, Jagielski, Lee, Tramer, and Zhang]{carlini2022quantifying-mem4}
Carlini, N., Ippolito, D., Jagielski, M., Lee, K., Tramer, F., and Zhang, C.
\newblock Quantifying memorization across neural language models.
\newblock \emph{arXiv preprint arXiv:2202.07646}, 2022.

\bibitem[Csord{\'a}s et~al.(2021)Csord{\'a}s, Irie, and Schmidhuber]{csordas2021devil}
Csord{\'a}s, R., Irie, K., and Schmidhuber, J.
\newblock The devil is in the detail: Simple tricks improve systematic generalization of transformers.
\newblock \emph{arXiv preprint arXiv:2108.12284}, 2021.

\bibitem[Dosovitskiy et~al.(2020)Dosovitskiy, Beyer, Kolesnikov, Weissenborn, Zhai, Unterthiner, Dehghani, Minderer, Heigold, Gelly, et~al.]{dosovitskiy2020image}
Dosovitskiy, A., Beyer, L., Kolesnikov, A., Weissenborn, D., Zhai, X., Unterthiner, T., Dehghani, M., Minderer, M., Heigold, G., Gelly, S., et~al.
\newblock An image is worth 16x16 words: Transformers for image recognition at scale.
\newblock \emph{arXiv preprint arXiv:2010.11929}, 2020.

\bibitem[Feldman \& Zhang(2020)Feldman and Zhang]{feldman2020neural-mem2}
Feldman, V. and Zhang, C.
\newblock What neural networks memorize and why: Discovering the long tail via influence estimation.
\newblock \emph{Advances in Neural Information Processing Systems}, 33:\penalty0 2881--2891, 2020.

\bibitem[Gulrajani \& Lopez-Paz(2020)Gulrajani and Lopez-Paz]{gulrajani2020search}
Gulrajani, I. and Lopez-Paz, D.
\newblock In search of lost domain generalization.
\newblock \emph{arXiv preprint arXiv:2007.01434}, 2020.

\bibitem[Gunasekar et~al.(2018{\natexlab{a}})Gunasekar, Lee, Soudry, and Srebro]{gunasekar2018characterizing}
Gunasekar, S., Lee, J., Soudry, D., and Srebro, N.
\newblock Characterizing implicit bias in terms of optimization geometry, 2018{\natexlab{a}}.
\newblock URL \url{https://arxiv.org/abs/1802.08246}.

\bibitem[Gunasekar et~al.(2018{\natexlab{b}})Gunasekar, Lee, Soudry, and Srebro]{gunasekar2018implicit}
Gunasekar, S., Lee, J., Soudry, D., and Srebro, N.
\newblock Implicit bias of gradient descent on linear convolutional networks, 2018{\natexlab{b}}.
\newblock URL \url{https://arxiv.org/abs/1806.00468}.

\bibitem[Hendrycks \& Gimpel(2016)Hendrycks and Gimpel]{hendrycks2016gaussian}
Hendrycks, D. and Gimpel, K.
\newblock Gaussian error linear units (gelus).
\newblock \emph{arXiv preprint arXiv:1606.08415}, 2016.

\bibitem[Hupkes et~al.(2020)Hupkes, Dankers, Mul, and Bruni]{hupkes2020compositionality}
Hupkes, D., Dankers, V., Mul, M., and Bruni, E.
\newblock Compositionality decomposed: How do neural networks generalise?
\newblock \emph{Journal of Artificial Intelligence Research}, 67:\penalty0 757--795, 2020.

\bibitem[Jacot et~al.(2021)Jacot, Ged, {\c{S}}im{\c{s}}ek, Hongler, and Gabriel]{jacot2021saddle}
Jacot, A., Ged, F., {\c{S}}im{\c{s}}ek, B., Hongler, C., and Gabriel, F.
\newblock Saddle-to-saddle dynamics in deep linear networks: Small initialization training, symmetry, and sparsity.
\newblock \emph{arXiv preprint arXiv:2106.15933}, 2021.

\bibitem[Johnson et~al.(2017)Johnson, Hariharan, Van Der~Maaten, Fei-Fei, Lawrence~Zitnick, and Girshick]{johnson2017clevr}
Johnson, J., Hariharan, B., Van Der~Maaten, L., Fei-Fei, L., Lawrence~Zitnick, C., and Girshick, R.
\newblock Clevr: A diagnostic dataset for compositional language and elementary visual reasoning.
\newblock In \emph{Proceedings of the IEEE conference on computer vision and pattern recognition}, pp.\  2901--2910, 2017.

\bibitem[Kandpal et~al.(2022)Kandpal, Wallace, and Raffel]{kandpal2022deduplicating-mem3}
Kandpal, N., Wallace, E., and Raffel, C.
\newblock Deduplicating training data mitigates privacy risks in language models.
\newblock In \emph{International Conference on Machine Learning}, pp.\  10697--10707. PMLR, 2022.

\bibitem[Lake \& Baroni(2018)Lake and Baroni]{lake2018generalization}
Lake, B. and Baroni, M.
\newblock Generalization without systematicity: On the compositional skills of sequence-to-sequence recurrent networks.
\newblock In \emph{International conference on machine learning}, pp.\  2873--2882. PMLR, 2018.

\bibitem[Lewkowycz et~al.(2022)Lewkowycz, Andreassen, Dohan, Dyer, Michalewski, Ramasesh, Slone, Anil, Schlag, Gutman-Solo, et~al.]{lewkowycz2022solving-minerva}
Lewkowycz, A., Andreassen, A., Dohan, D., Dyer, E., Michalewski, H., Ramasesh, V., Slone, A., Anil, C., Schlag, I., Gutman-Solo, T., et~al.
\newblock Solving quantitative reasoning problems with language models.
\newblock \emph{arXiv preprint arXiv:2206.14858}, 2022.

\bibitem[Loshchilov \& Hutter(2017)Loshchilov and Hutter]{loshchilov2017decoupled}
Loshchilov, I. and Hutter, F.
\newblock Decoupled weight decay regularization.
\newblock \emph{arXiv preprint arXiv:1711.05101}, 2017.

\bibitem[Mahdavi et~al.(2022)Mahdavi, Swersky, Kipf, Hashemi, Thrampoulidis, and Liao]{mahdavi2022better}
Mahdavi, S., Swersky, K., Kipf, T., Hashemi, M., Thrampoulidis, C., and Liao, R.
\newblock Towards better out-of-distribution generalization of neural algorithmic reasoning tasks.
\newblock \emph{ArXiv}, 2211.00692, 2022.
\newblock URL \url{https://arxiv.org/abs/2211.00692}.

\bibitem[Mansour et~al.(2009)Mansour, Mohri, and Rostamizadeh]{mansour2009domain}
Mansour, Y., Mohri, M., and Rostamizadeh, A.
\newblock Domain adaptation: Learning bounds and algorithms.
\newblock \emph{arXiv preprint arXiv:0902.3430}, 2009.

\bibitem[Miller et~al.(2021)Miller, Taori, Raghunathan, Sagawa, Koh, Shankar, Liang, Carmon, and Schmidt]{miller2021accuracy}
Miller, J.~P., Taori, R., Raghunathan, A., Sagawa, S., Koh, P.~W., Shankar, V., Liang, P., Carmon, Y., and Schmidt, L.
\newblock Accuracy on the line: on the strong correlation between out-of-distribution and in-distribution generalization.
\newblock In \emph{International Conference on Machine Learning}, pp.\  7721--7735. PMLR, 2021.

\bibitem[Moroshko et~al.(2020)Moroshko, Gunasekar, Woodworth, Lee, Srebro, and Soudry]{moroshko2020implicit}
Moroshko, E., Gunasekar, S., Woodworth, B., Lee, J.~D., Srebro, N., and Soudry, D.
\newblock Implicit bias in deep linear classification: Initialization scale vs training accuracy, 2020.
\newblock URL \url{https://arxiv.org/abs/2007.06738}.

\bibitem[Rahaman et~al.(2019)Rahaman, Baratin, Arpit, Draxler, Lin, Hamprecht, Bengio, and Courville]{rahaman2019spectral}
Rahaman, N., Baratin, A., Arpit, D., Draxler, F., Lin, M., Hamprecht, F., Bengio, Y., and Courville, A.
\newblock On the spectral bias of neural networks.
\newblock In \emph{International Conference on Machine Learning}, pp.\  5301--5310. PMLR, 2019.

\bibitem[Redko et~al.(2020)Redko, Morvant, Habrard, Sebban, and Bennani]{redko2020survey}
Redko, I., Morvant, E., Habrard, A., Sebban, M., and Bennani, Y.
\newblock A survey on domain adaptation theory: learning bounds and theoretical guarantees.
\newblock \emph{arXiv preprint arXiv:2004.11829}, 2020.

\bibitem[Saxton et~al.(2019)Saxton, Grefenstette, Hill, and Kohli]{saxton2019analysing}
Saxton, D., Grefenstette, E., Hill, F., and Kohli, P.
\newblock Analysing mathematical reasoning abilities of neural models.
\newblock \emph{arXiv preprint arXiv:1904.01557}, 2019.

\bibitem[Soudry et~al.(2017)Soudry, Hoffer, Nacson, Gunasekar, and Srebro]{soudry2017implicit}
Soudry, D., Hoffer, E., Nacson, M.~S., Gunasekar, S., and Srebro, N.
\newblock The implicit bias of gradient descent on separable data, 2017.
\newblock URL \url{https://arxiv.org/abs/1710.10345}.

\bibitem[Vaswani et~al.(2017)Vaswani, Shazeer, Parmar, Uszkoreit, Jones, Gomez, Kaiser, and Polosukhin]{vaswani2017attention}
Vaswani, A., Shazeer, N., Parmar, N., Uszkoreit, J., Jones, L., Gomez, A.~N., Kaiser, {\L}., and Polosukhin, I.
\newblock Attention is all you need.
\newblock \emph{Advances in neural information processing systems}, 30, 2017.

\bibitem[Veli{\v{c}}kovi{\'c} et~al.(2022)Veli{\v{c}}kovi{\'c}, Badia, Budden, Pascanu, Banino, Dashevskiy, Hadsell, and Blundell]{velivckovic2022clrs}
Veli{\v{c}}kovi{\'c}, P., Badia, A.~P., Budden, D., Pascanu, R., Banino, A., Dashevskiy, M., Hadsell, R., and Blundell, C.
\newblock The clrs algorithmic reasoning benchmark.
\newblock \emph{arXiv preprint arXiv:2205.15659}, 2022.

\bibitem[Wiles et~al.(2022)Wiles, Gowal, Stimberg, Rebuffi, Ktena, Dvijotham, and Cemgil]{wiles2022a}
Wiles, O., Gowal, S., Stimberg, F., Rebuffi, S.-A., Ktena, I., Dvijotham, K.~D., and Cemgil, A.~T.
\newblock A fine-grained analysis on distribution shift.
\newblock In \emph{International Conference on Learning Representations}, 2022.
\newblock URL \url{https://openreview.net/forum?id=Dl4LetuLdyK}.

\bibitem[Xu et~al.(2019)Xu, Zhang, Luo, Xiao, and Ma]{xu2019frequency}
Xu, Z.-Q.~J., Zhang, Y., Luo, T., Xiao, Y., and Ma, Z.
\newblock Frequency principle: Fourier analysis sheds light on deep neural networks, 2019.

\bibitem[Yun et~al.(2020)Yun, Krishnan, and Mobahi]{yun2020unifying}
Yun, C., Krishnan, S., and Mobahi, H.
\newblock A unifying view on implicit bias in training linear neural networks.
\newblock \emph{arXiv preprint arXiv:2010.02501}, 2020.

\bibitem[Zaremba \& Sutskever(2014)Zaremba and Sutskever]{zaremba2014learning}
Zaremba, W. and Sutskever, I.
\newblock Learning to execute.
\newblock \emph{arXiv preprint arXiv:1410.4615}, 2014.

\bibitem[Zhang et~al.(2021)Zhang, Raghu, Kleinberg, and Bengio]{Zhang2021PointerVR}
Zhang, C., Raghu, M., Kleinberg, J.~M., and Bengio, S.
\newblock Pointer value retrieval: A new benchmark for understanding the limits of neural network generalization.
\newblock \emph{ArXiv}, abs/2107.12580, 2021.

\bibitem[Zhang et~al.(2022)Zhang, Backurs, Bubeck, Eldan, Gunasekar, and Wagner]{zhang2022unveiling}
Zhang, Y., Backurs, A., Bubeck, S., Eldan, R., Gunasekar, S., and Wagner, T.
\newblock Unveiling transformers with lego: a synthetic reasoning task.
\newblock \emph{arXiv preprint arXiv:2206.04301}, 2022.

\end{thebibliography}
\bibliographystyle{icml2024}

\newpage
\appendix
\onecolumn
\section{Experiments Details}
\label{sec: exp details}
\subsection{Gradient Descent with Line Search}
\label{sec: gradient line search}
\begin{algorithm}
   \caption{Gradient Descent with Line Search}
   \label{gd_line_search}
\begin{algorithmic}
   \STATE {\bfseries Input:} data point $x_0$, Liphitz constant estimator $L_0 = 1$
   \FOR{$n=0,1,\ldots$}
   \STATE $x_{n+1} = x_n - \frac{1}{L_n} \nabla f(x_n)$
   \WHILE{not $f(x_{n+1}) \leq f(x_n) - \frac{1}{2 L_n} \| \nabla f(x_n) \|^2$}
   \STATE $L_n := 2 L_n$
   \STATE $x_{n+1} = x_n - \frac{1}{L_n} \nabla f(x_n)$
   \ENDWHILE
   \STATE $L_{n+1} = L_n / 2$
   \ENDFOR
\end{algorithmic}
\end{algorithm}

\subsection{Transformer Architecture}
\label{sec: transformer}
For Transformer, we use the encoder-only architecture from the Vision Transformer model \cite{dosovitskiy2020image}.
This model consists of 12 layers, each of them formed by multi-head self-attention block with 6 heads followed by Feed-Forward block.
Following standard practices, there is a layer normalization before each self-attention and Feed-Forward blocks. The model uses decoupled weights, i.e.  there is no parameters sharing between the layers or the attention heads.
For each input sequence, the model prepends a special classification token at the beginning of the sequence.
Then it encodes each token (which comes from the discrete alphabet) using the input embedding layer and adds it to the trainable positional embedding. We keep the embedding dimension equal to 64 both at the beginning of the model and between the model blocks.
The Feed-Forward module is represented by a 2-layers MLP with hidden dimension 64 and GELU activation \cite{hendrycks2016gaussian}. To get the final prediction, the model extracts the final classification token embedding, and feeds it through the layer normalization followed by a linear layer with a single output.

\section{Proof of Theorem \ref{theorem: 1}}
\label{sec:lemmas}

\paragraph{Reminder on the Hermite decomposition.} Let $H_t$ denote the Hermite polynomial of degree $t$, using the probabilist convention, and normalized such that $\{H_t, \, t \geq 0\}$ is an orthonormal basis of $L^2(\RR, \gamma_1)$. (We recall that $\gamma_d$ denotes the standard Gaussian measure in dimension~$d$.) Said differently, if $Z$ is a univariate standard normal random variable, we assume that $\E\left[H_s(Z) H_t(Z) \right] = \bfone_{s = t}$. Further, we define the multivariate Hermite polynomials as $\chi_T(x) = \prod_{i=1}^d H_{t_i}(x_i)$, where $T = (t_1, \ldots, t_d) \in \NN^d$.
Note that $\deg(\chi_T(x)) = |T| = T_1 + \dots + T_d$. The set of functions $\{\chi_T(x), T\in\NN^d\}$ forms an orthonormal basis of $L^2(\RR^d, \gamma_d)$.

Recall that $\Pi_p(\RR^d)$ denotes the set of polynomials of degree at most $p$ on $\RR^d$.
Any $h \in \Pi_p(\RR^d)$ admits a Hermite decomposition of the form $h(x) = \sum_{T\in\NN_{\leq p}^d} \widehat{h}(T) \chi_T(x)$,
where $\NN_{\leq p}^d = \{T\in \NN^d \, \vert \, |T| \leq p\}$ and 
$\widehat{h}(T) = \E_Z[h(Z) \chi_T(Z)]$, $Z \sim \cN(0, I_d)$.

We now turn to the proof of the theorem.

\begin{lemma} \label{lemma: 9}
    If $\sigma \in \Pi_p(\RR)$, then w.h.p.~we have that for any large enough $N$,
    \begin{gather*}
        \im(f_{\textnormal{RF}}) = \Pi_p(\RR^d) \, .
    \end{gather*}
\end{lemma}

\begin{proof}
    Since $\sigma \in \Pi_p(\RR)$, we have that $\forall i \in [N]$: $\phi_{w_i, b_i}(x) = \sigma(\innerproduct{w_i}{x}+b_i)\in \Pi_p(\RR^d)$. 
    Thus, $f_{\textnormal{RF}}(a;x) = \frac{1}{\sqrt{N}} \sum_{i=1}^N a_i \phi_{w_i,b_i}(x)\in \Pi_p(\RR^d)$ $\forall a\in \RR^N$, which shows that $\im(f_{\textnormal{RF}}) \subseteq \Pi_p(\RR^d)$.
    
    It remains to show that $\im(f_{\textnormal{RF}}) \supseteq \Pi_p(\RR^d)$.
    Let $\gamma_d$ denote the standard Gaussian measure in dimension $d$. We define the operator $M: \Pi_p(\RR^d) \to \Pi_p(\RR^d)$ by the formula 
    \begin{align*}
        &M(f) = \E_{w,b} \left[\left\langle f, \sigma(\langle w,x \rangle + b) \right\rangle_{L^2(\gamma_d)} \sigma(\langle w,x \rangle + b)\right] \, , \qquad f \in \Pi_d(\RR^d) \, .
    \end{align*}
    Similarly, we define the empirical version $M_N: \Pi_p(\RR^d) \to \Pi_p(\RR^d)$ by the formula
        \begin{align*}
        &M_N(f) = \frac{1}{N} \sum_{i=1}^N \left\langle f, \sigma(\langle w_i,x \rangle + b_i) \right\rangle_{L^2(\gamma_d)} \sigma(\langle w_i,x \rangle + b_i) \, , \qquad f \in \Pi_d(\RR^d) \, .
    \end{align*}
    The operators $M$ and $M_N$ are positive definite over $(\Pi_p(\RR^d), \langle ., . \rangle_{L^2(\gamma_d)})$. 

    \paragraph{}Assume that $M$ is positive definite. By the law of large numbers, $M_N \xrightarrow[N\to\infty]{} M$ almost surely. As the set of positive definite matrices is an open set, this implies that for all $\eta > 0$, there exists $N_0 \in \NN$ such that for all $N \geq N_0$, $\Pr(M_N\text{ is positive definite}) \geq 1-\eta$. 

    When $M_N$ is positive definite, then $\im(M_N) = \Pi_p(\RR^d)$. As $\im(M_N) \subset \im(f_{\textnormal{RF}}) \subset \Pi_p(\RR^d)$, this enables to conclude that $\im(f_{\textnormal{RF}}) = \Pi_p(\RR^d)$.

    \paragraph{}We thus now need to prove that $M$ is positive definite. Consider $f$ such that $\left\langle f, M(f) \right\rangle_{L^2(\gamma_d)} = 0$. We prove that $f=0$.

     We have that $\E_{w,b} \left[\left\langle f, \sigma(\langle w,x \rangle + b) \right\rangle_{L^2(\gamma_d)}^2\right] = 0$. Thus $\left\langle f, \sigma(\langle w,x \rangle + b) \right\rangle_{L^2(\gamma_d)} = 0$ almost surely. As this expression is a multivariate polynomial in $w$ and $b$, this implies that actually $\left\langle f, \sigma(\langle w,x \rangle + b) \right\rangle_{L^2(\gamma_d)} = 0$ for all $w \in \RR^d$ and $b \in \RR$. In particular, if $\lambda \geq 0$, we have $\left\langle f, \sigma(\lambda \langle w,x \rangle) \right\rangle_{L^2(\gamma_d)} = 0$. 

     We use a Taylor expansion for $\sigma$:
     \begin{equation*}
         \sigma(y) = \sum_{k=0}^p \frac{\sigma^{(k)}(0)}{k!} y^k \, .
     \end{equation*}
     As a consequence, 
     \begin{align*}
         0 = \left\langle f, \sigma(\lambda \langle w,x \rangle) \right\rangle_{L^2(\gamma_d)} = \sum_{k=0}^p \left\langle f , \frac{\sigma^{(k)}(0)}{k!} \left(\lambda \langle w, x \rangle \right)^k\right\rangle_{L^2(\gamma_d)} =  \sum_{k=0}^p\frac{\sigma^{(k)}(0)}{k!} \lambda^k \left\langle f , \langle w, x \rangle^k\right\rangle_{L^2(\gamma_d)} \, .
     \end{align*}
     Identifying powers of $\lambda$ in this expression, and using that $\sigma^{(k)}(0) \neq 0$ for all $k$, we have that 
     \begin{align*}
         \left\langle f , \langle w, x \rangle^k\right\rangle_{L^2(\gamma_d)} \, , \qquad k = 0, \dots, p \, .
     \end{align*}
     To conclude, we only need to prove that the set of functions $\langle w, x \rangle^k$ for $w \in \RR^d$ and $k = 0, \dots, p$ spans $\Pi_p(\RR^d)$. 

     Consider $w$ such that $\Vert w \Vert = 1$. We decompose $f$ into multivariate Hermite polynomials: $f(x) = \sum_{|\beta| \leq p} \widehat{f}(\beta) h_\beta(x)$. Then 
     \begin{align*}
         0 = \left\langle f , h_k(\langle w, x \rangle)\right\rangle_{L^2(\gamma_d)} = \sum_{|\beta| \leq p} \widehat{f}(\beta) \left\langle h_\beta(x) , h_k(\langle w, x \rangle)\right\rangle_{L^2(\gamma_d)} = \sum_{|\beta| \leq p} \widehat{f}(\beta) {k \choose \beta}^{1/2} w_1^{\beta_1} \cdots w_d^{\beta_d} \, .
     \end{align*}
     The last quantity is a multivariate polynomial in $w$, which is zero on the unit sphere. It thus need to be identically zero. Thus for all $\beta$, $\widehat{f}(\beta) =0$. Thus $f=0$. This concludes the proof.
\end{proof}

\begin{lemma} \label{lemma: 8}
    Assume that $A_n \rightarrow A$ as $n \rightarrow \infty$, where $(A_n)_{n=1}^{\infty}$, $A$ are positive-definite matrices in $\RR^{d\times d}$ and let $\mathcal{S}$ be any affine subspace of $\RR^d$.
    Then
    \begin{gather*}
        \text{argmin}_{x\in\mathcal{S}} x^\top A_n x \rightarrow \text{argmin}_{x\in \mathcal{S}} x^\top A x
    \end{gather*}
\end{lemma}
\begin{proof}
    Let us introduce the following notations:
    \begin{gather*}
        y_n = \text{argmin}_{x\in\mathcal{S}} x^\top A_n x \\
        y = \text{argmin}_{x\in\mathcal{S}} x^\top A x \\
        \rho = \min_{x\in\mathcal{S}} x^\top A x = y^\top A y
    \end{gather*}
    First, let us show that $\|y_n\|_2$ is uniformly bounded over $n$.
    From $A_n \rightarrow A$ we get that $\lambda_{min}(A_n) \rightarrow \lambda_{min}(A)$ and $y^\top A_n y \rightarrow y^\top A y$.
    Thus, $\exists n_0 \in \NN$ s.t. $\forall n \geq n_0$ both of the following holds: $\lambda_{min}(A_n) \geq \lambda_{min}(A)/2 > 0$ and $| y^\top A_n y - y^\top A y| < 1$.
    We claim that for any such $n$ it holds that $\|y_n\|_2 \leq C \overset{def}{=} \sqrt{\frac{2(\rho+1)}{\lambda_{min}(A)}}$.
    Indeed, consider any $x$ such that $\|x\|_2 > C$. Then
    \begin{gather*}
        x^\top A_n x \geq \lambda_{min}(A_n) \|x\|_2^2 \geq \frac{\lambda_{min}(A)}{2} \|x\|_2^2 
        > \frac{\lambda_{min}(A)}{2} \frac{2(\rho+1)}{\lambda_{min}(A)} = \rho + 1
    \end{gather*}
    On the other hand, $y^\top A_n y < y^\top A y + 1 = \rho + 1$.
    Hence, $y^\top A_n y < x^\top A_n x$, which shows that none of such $x$ can be the minimizer.
    This concludes $\|y\|_2 \leq C$.
    As a side effect, from $y^\top A_n y < x^\top A_n x$ $\forall x: \|x\|_2 > C$ we get $\|y\|_2 \leq C$.

    To complete the proof, it is enough to show that $y$ is a unique partial limit of the sequence $(y_n)$. 
    Define $\mathcal{S}' = \{x \in \mathcal{S}: \|x\|_2 \leq C\}$.
    As we have proved above, $y, y_n \in \mathcal{S}'$ for large enough $n$.
    Let us show that the functions $x \rightarrow x^\top A_n x$ converge to $x \rightarrow x^\top A x$ uniformly over $x \in \mathcal{S}'$.
    Indeed,
    \begin{gather*}
        |x^\top A_n x - x^\top A x| \leq \|x\|_2 \cdot \|(A_n-A)x\|_2 \leq \|A_n-A\|_2 \cdot \|x\|_2^2
    \end{gather*}
    and the last term converges to zero uniformly whenever $\|x\|_2$ is uniformly bounded.

    Now let $l$ be any partial limit of $y_n$, that is $\exists(y_{n_k})$: $y_{n_k} \rightarrow l$.
    We want to show that $l = y$.
    Let us fix any $\varepsilon>0$. 
    From the uniform convergence we know that $\exists n_0\in \NN$ s.t. $|x^\top A_n x - x^\top A x| < \varepsilon$ $\forall x \in \mathcal{S}'$, $\forall n \geq n_0$.
    Recall that $y, y_n \in \mathcal{S}'$ for large enough $n$, which means all the elements of $(y_{n_k})$ belong to $\mathcal{S}'$ starting from some $k_0$.
    For these elements, we can estimate
    \begin{gather*}
        \rho = y^\top A y \leq y_{n_k}^\top A y_{n_k} \leq y_{n_k}^\top A_{n_k} y_{n_k} + \varepsilon \leq y^\top A_{n_k} y + \varepsilon 
        \leq y^\top A y + 2 \varepsilon = \rho + 2 \varepsilon
    \end{gather*}
    where the first inequality comes from $y = \text{argmin}_{x\in\mathcal{S}} x^\top A x$, and the third - from $y_{n_k} = \text{argmin}_{x\in\mathcal{S}} x^\top A_{n_k} x$.
    Thus, $\forall \varepsilon>0$ we get
    \begin{gather*}
        \rho \leq y_{n_k}^\top A y_{n_k} \leq \rho + 2 \varepsilon \quad \forall k\geq k_0(\varepsilon)
    \end{gather*}
    which shows that $y_{n_k}^\top A y_{n_k} \rightarrow \rho = \min_{x\in\mathcal{S}} x^\top A x$. 
    Taking into account that the function $x \rightarrow x^\top A x$ is strongly convex (since $A$ is positive-definite), we conclude that $y_{n_k} \rightarrow y$.
    Hence, the only partial limit of $(y_n)$ is $y$, which proves that $y_n \rightarrow y$.
\end{proof}

\begin{lemma} \label{lemma: 2}
    For any basis monomial $\chi_T(x)$ and any non-negative integer $k < |T|$ we have:
    \begin{gather*}
        \mathbb{E}_{x\sim\mathcal{N}(0, I_d)}[(\innerproduct{w}{x}+b)^k \chi_T(x)] = 0
    \end{gather*}
\end{lemma}
\begin{proof}
    Since the term $(\innerproduct{w}{x}+b)^k$ is a polynomial of degree $k$ in $x$, in Hermite basis it only contains Hermite polynomials of degree at most $k$. The rest comes from the orthogonality of Hermite polynomials.
\end{proof}

\begin{lemma} \label{lemma: 3}
    For any non-negative integer $k$ we have:
    \begin{gather*}
        \mathbb{E}_{x\sim\mathcal{N}(0, I_d)} [| (\innerproduct{w}{x}+b)^k \chi_T(x) |] \leq \varepsilon^{k/2} \text{poly}_T(\Bar{w}, \Bar{b})
    \end{gather*}
    where $\text{poly}_T(\Bar{w}, \Bar{b})$ is some polynomial in $\Bar{w}_1, \ldots, \Bar{w}_d, \Bar{b}$ which depends on $T$.
    Here, $\Bar{w} = \varepsilon^{-1/2}w$, $\Bar{b} = \varepsilon^{-1/2} b$.
\end{lemma}
\begin{proof}
    \begin{gather*}
        \mathbb{E}_x [| (\innerproduct{w}{x}+b)^k \chi_T(x) |]
        = \varepsilon^{k/2} \mathbb{E}_x [| (\innerproduct{\Bar{w}}{x}+\Bar{b})^k \chi_T(x) |] \\
        = \varepsilon^{k/2} \mathbb{E}_x \left[ \left| \sum_{\alpha_1+\ldots+\alpha_{d+1}=k} \binom{k}{\alpha_1 \ldots \alpha_{d+1}} (\Bar{w}_1 x_1)^{\alpha_1} \ldots (\Bar{w}_d x_d)^{\alpha_d} \Bar{b}^{\alpha_{d+1}} \chi_T(x) \right| \right] \\
        \leq \varepsilon^{k/2} \sum_{\alpha_1+\ldots+\alpha_{d+1}=k} \binom{k}{\alpha_1 \ldots \alpha_{d+1}} |\Bar{w}_1^{\alpha_1} \ldots \Bar{w}_d^{\alpha_d} \Bar{b}^{\alpha_{d+1}}| \cdot \mathbb{E}_x [|x_1^{\alpha_1} \ldots x_d^{\alpha_d} \chi_T(x)|] \\
        \leq \varepsilon^{k/2} \sum_{\alpha_1+\ldots+\alpha_{d+1}=k} \binom{k}{\alpha_1 \ldots \alpha_{d+1}} \left(\frac{\Bar{w}_1^2+1}{2}\right)^{\alpha_1} \ldots \left(\frac{\Bar{w}_d^2+1}{2}\right)^{\alpha_d} \left(\frac{\Bar{b}^2+1}{2}\right)^{\alpha_{d+1}} \cdot \mathbb{E}_x [|x_1^{\alpha_1} \ldots x_d^{\alpha_d} \chi_T(x)|] \\
        = \varepsilon^{k/2} \text{poly}_T(\Bar{w}, \Bar{b})
    \end{gather*}
\end{proof}

\begin{lemma} \label{lemma: 4}
    For any $T\in \NN^d$ such that $|T| \leq p$ we have
    \begin{gather*}
        \widehat{\phi}_{w,b}(T) = \varepsilon^{|T|/2} \cdot \frac{\sigma^{(|T|)}(0)}{|T|!} \binom{|T|}{t_1 \ldots t_d} \Bar{w}_1^{t_1} \ldots \Bar{w}_d^{t_d} + O(\varepsilon^{(|T|+1)/2} \cdot \text{poly}_T(\Bar{w}, \Bar{b}))
    \end{gather*}
    Here, $\Bar{w} = \varepsilon^{-1/2}w$, $\Bar{b} = \varepsilon^{-1/2} b$, so that the distribution of $\Bar{w}$ and $\Bar{b}$ does not depend on $\varepsilon$: $\Bar{w} \sim \mathcal{N}(0, I_d)$, $\Bar{b} \sim \mathcal{N}(0,1)$.
    
\end{lemma}
\begin{proof}
    Using Taylor series, we get
    \begin{gather*}
        \phi_{w,b}(x) = \sigma(\innerproduct{w}{x}+b) = \sigma(G)
        = \sigma(0) + \sigma'(0) G + \ldots + \frac{\sigma^{(|T|)}(0)}{|T|!} G^{|T|} + \frac{\sigma^{(|T|+1)}(\xi)}{(|T|+1)!} G^{|T|+1}
    \end{gather*}
    where $|\xi| \leq |G|$.
    Note that the expansion $\phi_{w,b}(x) = \sum_{T} \widehat{\phi}_{w,b}(T) \chi_T(x)$ is the basis of Hermite polymonials $\chi_T(x)$ does not depend on the distribution on the input space $\RR^d$.
    Hence, for simplicity, we can assume the Gaussian distribution: $x \sim \mathcal{N}(0, I_d)$, which allows us to express the Hermite coefficients using the dot product: $\widehat{\phi}_{w,b}(T) = \mathbb{E}_x[\phi_{w,b}(x) \chi(T)]$.
    Thus, by Lemma \ref{lemma: 2} we obtain
    \begin{gather*}
        \widehat{\phi}_{w,b}(T)
        = \mathbb{E}_x[\frac{\sigma^{(|T|)}(0)}{|T|!} G^{|T|} \chi_T(x) + \frac{\sigma^{(|T|+1)}(\xi)}{(|T|+1)!} G^{|T|+1} \chi_T(x)]
        = \mathbb{E}_x[A] + \mathbb{E}_x[B]
    \end{gather*}
    where $A = \frac{\sigma^{(|T|)}(0)}{|T|!} G^{|T|} \chi_T(x)$ and $B = \frac{\sigma^{(|T|+1)}(\xi)}{(|T|+1)!} G^{|T|+1} \chi_T(x)$.
    For the first term we get
    \begin{gather*}
        \mathbb{E}_x[A] 
        = \frac{\sigma^{(|T|)}(0)}{|T|!} \mathbb{E}_x[(\innerproduct{w}{x}+b)^{|T|} \chi_T(x)]
        = \varepsilon^{|T|/2} \frac{\sigma^{(|T|)}(0)}{|T|!} \mathbb{E}_x[(\innerproduct{\Bar{w}}{x}+\Bar{b})^{|T|} \chi_T(x)] \\
        = \varepsilon^{|T|/2} \frac{\sigma^{(|T|)}(0)}{|T|!} \binom{|T|}{t_1 \ldots t_d} \Bar{w}_1^{t_1} \ldots \Bar{w}_d^{t_d}
    \end{gather*}
    For the second term, we can estimate
    \begin{gather*}
        |\mathbb{E}_x[B]| \leq \mathbb{E}_x[|B|]
        = \frac{1}{(|T|+1)!} \mathbb{E}_x[| \sigma^{(|T|+1)}(\xi) G^{|T|+1} \chi_T(x) |]
    \end{gather*}
    By assumption \ref{assumption: 3} we have
    \begin{gather*}
        |\sigma^{(|T|+1)}(\xi)| \leq C (\xi^l + 1) \leq C (|\xi|^l + 1) \leq C (|G|^l + 1)
    \end{gather*}
    Substituting, we obtain:
    \begin{gather*}
        |\mathbb{E}_x[B]| 
        \leq \frac{C}{(|T|+1)!} \left( \mathbb{E}_x[| G^{|T|+1} \chi_T(x) |] + \mathbb{E}_x[| G^{|T|+l+1} \chi_T(x) |] \right)
    \end{gather*}
    Applying Lemma \ref{lemma: 3} to the expectations above, we proceed
    \begin{align*}
        |\mathbb{E}_x[B]| 
        & \leq \frac{C}{(|T|+1)!} \left( \varepsilon^{(|T|+1)/2} \text{poly}_T^{(1)}(\Bar{w}, \Bar{b}) + \varepsilon^{(|T|+l+1)/2} \text{poly}_T^{(2)}(\Bar{w}, \Bar{b}) \right) \\
        & \leq \frac{C}{(|T|+1)!} \cdot \varepsilon^{(|T|+1)/2} \left( \text{poly}_T^{(1)}(\Bar{w}, \Bar{b}) + \text{poly}_T^{(2)}(\Bar{w}, \Bar{b}) \right) \\
        & = \varepsilon^{(|T|+1)/2} \text{poly}_T(\Bar{w}, \Bar{b})
    \end{align*}
    which completes the proof.
\end{proof}

\begin{lemma} \label{lemma: 5}
    If $g \in \Pi_p(\RR^d)$, then
    \begin{gather*}
        \widehat{g}_{\leq p}^\top \Phi^{-1} \widehat{g}_{\leq p} = \Theta(\varepsilon^{-def(g)})
    \end{gather*}
\end{lemma}
\begin{proof}
    By the previous lemma we know that
    \begin{gather*}
        \Phi = \Gram \left\{ \widehat{\phi}_{w,b}(T) = \varepsilon^{|T|/2} c_T \Bar{w}_1^{t_1} \ldots \Bar{w}_d^{t_d} + O(\varepsilon^{(|T|+1)/2} \cdot \text{poly}_T(\Bar{w}, \Bar{b})), \; T \in \NN_{\leq p}^d \right\}
    \end{gather*}
    where $c_T = \frac{\sigma^{(|T|)}(0)}{|T|!} \binom{|T|}{t_1 \ldots t_d} \neq 0$ given that $\sigma^{(|T|)}(0) \neq 0$ by Assumption \ref{assumption: 3}.
    Define
    \begin{gather*}
        A = \Gram \left\{ \varepsilon^{- |T|/2}\widehat{\phi}_{w,b}(T) = c_T \Bar{w}_1^{t_1} \ldots \Bar{w}_d^{t_d} + O(\varepsilon^{1/2} \cdot \text{poly}_T(\Bar{w}, \Bar{b})), \; T \in \NN_{\leq p}^d \right\}
    \end{gather*}
    Then $A$ and $\Phi$ are connected by
    \begin{gather}
        \Phi_{i,j} = \varepsilon^{(|T_i|+|T_j|)/2} \cdot A_{i,j} \label{14} \\
        \Phi_{i,j}^{-1} = \varepsilon^{-(|T_i|+|T_j|)/2} \cdot A_{i,j}^{-1} \label{15}
    \end{gather}
    (the second can be established, for example, by Cramer's rule).
    Next, define 
    \begin{gather*}
        \Tilde{A} = \Gram \left\{ c_T \Bar{w}_1^{t_1} \ldots \Bar{w}_d^{t_d}, \; T \in \NN_{\leq p}^d \right\}
    \end{gather*}
    e.g. we dropped the reminders from the Gram basis elements of $A$.
    Then we have that $A_{i,j} = \Tilde{A}_{i,j} + O(\varepsilon^{1/2})$ $\forall i,j$ (here, we use that the expectation of any polynomial in $\Bar{w}, \Bar{b}$ is finite).
    It gives us $\det(A) = \det(\Tilde{A}) + O(\varepsilon^{1/2})$.
    Besides, denoting by $C$ and $\Tilde{C}$ the cofactor matrices of $A$ and $\Tilde{A}$ respectively, we also have
    $C_{i,j} = \Tilde{C}_{i,j} + O(\varepsilon^{1/2})$ $\forall i,j$.
    Finally, $\det(\Tilde{A}) \neq 0$ since the basis elements of $\Tilde{A}$ are linearly independent. 
    Combining all together, we have
    \begin{gather*}
        A_{i,j}^{-1} = \frac{C_{j,i}}{\det(A)} = \frac{\Tilde{C}_{j,i} + O(\varepsilon^{1/2})}{\det(\Tilde{A}) + O(\varepsilon^{1/2})} = \Tilde{A}_{i,j}^{-1} + O(\varepsilon^{1/2})
    \end{gather*}
    Combining with (\ref{15}), we get
    \begin{gather} \label{12}
        \Phi_{i,j}^{-1} = \varepsilon^{-(|T_i|+|T_j|)/2} \cdot A_{i,j}^{-1}
        = \varepsilon^{-(|T_i|+|T_j|)/2} \cdot \Tilde{A}_{i,j}^{-1} + O(\varepsilon^{-(|T_i|+|T_j|-1)/2})
    \end{gather}
    As a corollary, we may estimate
    \begin{gather} \label{11}
        \Phi_{i,j}^{-1} = O(\varepsilon^{-(|T_i|+|T_j|)/2})
    \end{gather}

    Now consider any fixed $g \in \Pi_p(\RR^d)$. Denote $s = \deg(g)$, then $\widehat{g}(T) = 0$ $\forall T: |T| > s$.
    Hence,
    \begin{gather*}
        \widehat{g}_{\leq p}^\top \Phi^{-1} \widehat{g}_{\leq p}
        = \sum_{|T|,|T'|\leq s} \widehat{g}(T) \widehat{g}(T') \Phi_{T,T'}^{-1}
    \end{gather*}
    Note that if $|T| < s$ or $|T'| < s$ then $(|T|+|T'|)/2 \leq s-1/2$ and from (\ref{11}) we get $\Phi_{T,T'}^{-1} = O(\varepsilon^{-s+1/2})$.
    Thus, we can estimate
    \begin{gather} \label{13}
        \widehat{g}_{\leq p}^\top \Phi^{-1} \widehat{g}_{\leq p}
        = \sum_{|T|,|T'| = s} \widehat{g}(T) \widehat{g}(T') \Phi_{T,T'}^{-1} + O(\varepsilon^{-s+1/2})
        \overset{(\ref{12})}{=} \varepsilon^{-s} \sum_{|T|,|T'| = s} \widehat{g}(T) \widehat{g}(T') \Tilde{A}_{T,T'}^{-1} + O(\varepsilon^{-s+1/2})
    \end{gather}
    Now define $g'\in \Pi_p(\RR^d)$ by setting $\widehat{g}'(T) = \widehat{g}(T)$ if $|T| = s$ and $\widehat{g}'(T) = 0$ otherwise.
    Then
    \begin{gather} \label{16}
        \sum_{|T|,|T'| = s} \widehat{g}(T) \widehat{g}(T') \Tilde{A}_{T,T'}^{-1} = (\widehat{g}'_{\leq p})^\top \Tilde{A}^{-1} \widehat{g}'_{\leq p}
    \end{gather}
    Note that $\Tilde{A} \succ 0$ since $A$ is a Gram matrix of linearly independent set of functions. Thus, $\Tilde{A}^{-1} \succ 0$.
    Moreover, since $deg(g) = s$, we have $\widehat{g}'_{\leq p} \neq 0$. Combining these, we conclude that the value of (\ref{16}) is strictly positive. Denoting this value by $C_1 > 0$ and substituting it into (\ref{13}), we obtain
    \begin{gather*}
        \widehat{g}_{\leq p}^\top \Phi^{-1} \widehat{g}_{\leq p} = C_1 \varepsilon^{-s} + O(\varepsilon^{-s+1/2}) = \Theta(\varepsilon^{-s})
    \end{gather*}
    which completes the proof.
\end{proof}

\begin{corollary} \label{corollary: 2}
    There exist $c, \varepsilon_0 > 0$ such that $\forall \varepsilon < \varepsilon_0$ we have:
    \begin{gather*}
        \Phi^{-1} \succeq c D_{\varepsilon}
    \end{gather*}
    where $D_{\varepsilon} = \diag(\{ \varepsilon^{-|T_i|}, T_i \in \NN_{\leq p}^d \})$
\end{corollary}
\begin{proof}
    We can write (\ref{15}) in matrix form as
    \begin{gather}
        \Phi^{-1} = D_{\varepsilon}^{1/2} A^{-1} D_{\varepsilon}^{1/2} \label{17}
    \end{gather}
    Since $\Tilde{A}^{-1} \succ 0$, it holds that $\Tilde{A}^{-1} \succ c I$ where $c = \lambda_{min}(\Tilde{A}^{-1})/2$.
    Combining with $A^{-1} \rightarrow \Tilde{A}^{-1}$, we obtain that $A^{-1} \succ c I$ for small enough $\varepsilon$.
    Substituting in (\ref{17}), we proceed
    \begin{gather*}
        \Phi^{-1} = D_{\varepsilon}^{1/2} A^{-1} D_{\varepsilon}^{1/2} \succeq D_{\varepsilon}^{1/2} (c I) D_{\varepsilon}^{1/2} = c D_{\varepsilon}
    \end{gather*}
    which completes the proof.
\end{proof}

\begin{proof}[Proof of Theorem \ref{theorem: 1}]
    The first statement of the theorem is given by Lemma \ref{lemma: 9}. We now turn to the proof of the second statement. Define the set of polynomial interpolators $\mathcal{F}_{\textnormal{int}}$ as 
\begin{gather*}
    \mathcal{F}_{\textnormal{int}} = \left\{h \in \Pi_p(\RR^d): \forall x \in \mathcal{U}^c, \, h(x) = f(x) \right\} \, .
\end{gather*}
Since $f\in\Pi_p(\RR^d)$, we have $\mathcal{F}_{\textnormal{int}} \neq \varnothing$.
Define the matrix $F\in \RR^{|\NN_{\leq p}^d|\times N}$ by setting $F_{ij} = \frac{1}{\sqrt{N}} \widehat{\phi}_{w_j, b_j}(T_i)$, where $j \in \{1, \dots, N\}$ and $T_1, \ldots, T_{|\NN_{\leq p}^d|}$ are enumerated elements of $\NN_{\leq p}^d$. Then the Hermite coefficients of the random features model can be expressed as $\widehat{f}_{\textnormal{RF}}(a) = F a$.

Consider $N$ large enough so that any interpolator $g \in \mathcal{F}_{\textnormal{int}}$ can be expressed by the random features model (such $N$ exists w.h.p. by Lemma \ref{lemma: 9}).
Then the equation
\begin{gather} \label{6}
    F a = \widehat{g}
\end{gather}
has solution in $a$ for any $g \in \mathcal{F}_{\textnormal{int}}$.
Moreover, provided that the matrix $F F^\top \in \RR^{|\NN_{\leq p}^d| \times |\NN_{\leq p}^d|}$ is invertible, the minimum-norm solution $a(g)$ of (\ref{6}) is given by
\begin{gather} \label{19}
    a(g) = F^{\dagger} \widehat{g} = F^\top (F F^\top)^{-1} \widehat{g}
\end{gather}
and we get
\begin{gather} \label{9}
    \|a(g)\|^2 = \widehat{g}^\top (F F^\top)^{-1} \widehat{g} \, .
\end{gather}
Let us show that $F F^\top$ is indeed invertible (w.h.p.).
We have
\vspace{-.5cm}
\begin{align*}
    (F F^\top)_{i,j}
    &= \sum_{k=1}^{N} F_{i,k} F^\top_{k,j}
    = \sum_{k=1}^{N} F_{i,k} F_{j,k} \\
    &= \frac{1}{N} \sum_{k=1}^N \widehat{\phi}_{w_k, b_k}(T_i) \widehat{\phi}_{w_k, b_k}(T_j) \\
    &\xrightarrow[N \to \infty]{\textnormal{a.s.}} \mathbb{E}_{w,b}[\widehat{\phi}_{w,b}(T_i) \widehat{\phi}_{w,b}(T_j)]
\end{align*}
where the last step follows from the Strong Law of Large Numbers (SLLN).
To be able to use the SLLN, we have to check that $\mathbb{E}_{w,b}[|\widehat{\phi}_{w,b}(T_i) \widehat{\phi}_{w,b}(T_j)|] < \infty$.
For this, we use the Cauchy-Schwarz inequality:
\begin{align} 
    &\mathbb{E}_{w,b}[|\widehat{\phi}_{w,b}(T_i) \widehat{\phi}_{w,b}(T_j)|] \label{20}\\
    &\qquad\leq \sqrt{\mathbb{E}_{w,b} [\widehat{\phi}_{w,b}(T_i)^2]} \cdot \sqrt{\mathbb{E}_{w,b} [\widehat{\phi}_{w,b}(T_j)^2]} \, . \label{20-bis}
\end{align}
The finiteness of the right-hand side follows (at least for small enough $\varepsilon$) from Lemma \ref{lemma: 4} after noting that the expectation of any polynomial in $\Bar{w}, \Bar{b}$ is finite, where $\Bar{w} \sim \mathcal{N}(0,I_d)$, $\Bar{b} \sim \mathcal{N}(0,1)$.
In the following, we consider $\varepsilon$ small enough for \eqref{20}--\eqref{20-bis} to hold for any $i, j$.

Thus, we get that $F F^\top \overset{a.s.}{\rightarrow} \Phi$, where $\Phi \in \RR^{|\NN_{\leq p}^d| \times |\NN_{\leq p}^d|}$ is a deterministic matrix defined by 
\begin{gather} \label{2}
    \Phi_{ij} = \mathbb{E}_{w,b}[\widehat{\phi}_{w,b}(T_i) \widehat{\phi}_{w,b}(T_j)]
\end{gather}
Let us show that the matrix $\Phi$ is invertible.
Note that $\Phi$ is the Gram matrix for the set of functions $\{ (w,b) \mapsto \widehat{\phi}_{w,b}(T), \; T \in \NN_{\leq p}^d \}$ in $L^2(\RR^{d+1}, \gamma_{d+1})$ space.
Hence, if matrix $\Phi$ were degenerate, it would mean that the functions $\{ (w,b) \mapsto \widehat{\phi}_{w,b}(T), \; T \in \NN_{\leq p}^d \}$ are linearly dependent.
Denote $k = |\NN_{\leq p}^d|$.
Then there exists a linear subspace $L$ of dimension $\dim(L) \leq k-1$ such that for all $w, b$ we  have $\widehat{\phi}_{w,b} \in L$. This implies that $\widehat{f}_{\textnormal{RF}} \in L$ for all $N$, $\{w_i\}_{i=1}^N, \{b_i\}_{i=1}^N, \{a_i\}_{i=1}^N$. 
Therefore, we have $\dim(\im(\widehat{f}_{\textnormal{RF}})) \leq k-1$ thus
$\dim(\im(f_{\textnormal{RF}})) = \dim(\im(\widehat{f}_{\textnormal{RF}})) \leq k-1 < k = \dim(\Pi_p(\RR^d))$. 
The last inequality shows that $\im(f_{\textnormal{RF}}) \neq \Pi_p(\RR^d)$, and this statement holds for all $N$, $\{w_i\}_{i=1}^N, \{b_i\}_{i=1}^N$.
Thus, we get a contradiction with Lemma \ref{lemma: 9}.
This proves that matrix $\Phi$ must be invertible.

Since $\Phi$ is invertible and $F F^\top \rightarrow \Phi$ as $N\rightarrow \infty$, the matrix $F F^\top$ must be invertible for large enough $N$ and $(F F^\top)^{-1} \rightarrow \Phi^{-1}$.
Thus, we justified (\ref{19})--(\ref{9}) and from (\ref{9}) can deduce
\begin{gather} \label{10}
    \|a(g)\|^2 \xrightarrow[N\to\infty]{a.s.} \widehat{g}^\top \Phi^{-1} \widehat{g} \, .
\end{gather}

Recall that $a^*$ denotes the minimum norm interpolating solution. Thus, for finite $N$, $f_{\textnormal{RF}}(a^*)$ is the minimizer of~(\ref{9}) over $g \in \mathcal{F}_{\textnormal{int}}$.
Besides, denote the minimizer of (\ref{10}) over $g \in \mathcal{F}_{\textnormal{int}}$ by $g_{\varepsilon}$.
Since $(F F^\top)^{-1} \succ 0$ (for large enough $N$), $\Phi^{-1} \succ 0$, $(F F^\top)^{-1} \rightarrow \Phi^{-1}$ as $N \to \infty$, and since $\mathcal{F}_{\textnormal{int}}$ is an affine subspace, by Lemma \ref{lemma: 8} we get that $\widehat{f}_{\textnormal{RF}}(a^*) \rightarrow \widehat{g}_{\varepsilon}$ as $N\to \infty$ for any small enough $\varepsilon > 0$,
which implies ${f}_{\textnormal{RF}}(a^*) \rightarrow g_{\varepsilon}$.

It remains to show that $\dist(g_{\varepsilon}, \Pi_{p_*}) \rightarrow 0$ as $\varepsilon \rightarrow 0$.
Consider $h \in \mathcal{F}_{\textnormal{int}}$ such that $\deg(h) = p_*$.
Then by Lemma \ref{lemma: 5} we have $\widehat{h}^\top \Phi^{-1} \widehat{h} = \Theta(\varepsilon^{-p_*})$, which implies $\exists c_1>0$: $\widehat{h}^\top \Phi^{-1} \widehat{h} \leq c_1 \varepsilon^{-p_*}$ for any small enough $\varepsilon>0$.
Since $g_{\varepsilon}$ is the minimizer of (\ref{10}), we can estimate
\begin{gather}\label{18}
    \widehat{g}_{\varepsilon}^\top \Phi^{-1} \widehat{g}_{\varepsilon} \leq \widehat{h}^\top \Phi^{-1} \widehat{h} \leq c_1 \varepsilon^{-p_*}
\end{gather}

On the other hand, from Corollary \ref{corollary: 2}, there exists $c_2 > 0$ such that 
\begin{align*}
    \widehat{g}_{\varepsilon}^\top \Phi^{-1} \widehat{g}_{\varepsilon} 
    &\geq  c_2 \sum_{|T|\leq p} \widehat{g}_{\varepsilon}(T)^2 \varepsilon^{-|T|}  \\
    &\geq c_2 \sum_{|T| = k} \widehat{g}_{\varepsilon}(T)^2 \varepsilon^{-k}
    = c_2 e_{\varepsilon}(k) \varepsilon^{-k} 
\end{align*}
where we define $e_{\varepsilon}(k) = \sum_{|T| = k} \widehat{g}_{\varepsilon}(T)^2$ - the energy of the degree-$k$ monomials of $g_{\varepsilon}$.
Combining this with (\ref{18}), we obtain
\begin{gather*}
    c_2 e_{\varepsilon}(k) \varepsilon^{-k} \leq \widehat{g}_{\varepsilon}^\top \Phi^{-1} \widehat{g}_{\varepsilon} \leq c_1 \varepsilon^{-p_*} \Rightarrow
    e_{\varepsilon}(k) \leq \frac{c_1}{c_2} \varepsilon^{k-p_*} \, .
\end{gather*}
For $k > p_*$ we have $\varepsilon^{k-p_*} \rightarrow 0$, and thus $e_{\varepsilon}(k) \to 0$. As $\dist(g_\varepsilon, \Pi_{p_*}(\RR^d))$ is bounded by $\sum_{k>p_*} e_\varepsilon(k)$ up to a constant, this concludes the proof.
\end{proof}

\section{Proof of Example \ref{example: 1}}
\label{sec:example proof}

\begin{lemma} \label{lemma: 1}
    Let $T = (t_1, \ldots, t_d), T' = (t_1', \ldots, t_d') \in \NN^d$ are such that $\exists i\in[d]$: $t_i \not\equiv t_i'$ (mod 2).
    Then it holds
    \begin{gather*}
        \mathbb{E}_{w,b}[\widehat{\phi}_{w,b}(T) \widehat{\phi}_{w,b}(T')] = 0
    \end{gather*}
\end{lemma}
\begin{proof}
    Denote by $w_{-i}$ and $x_{-i}$ the vectors $w$ and $x$ respectively with flipped $i$-th coordinate.
    Note that
    \begin{gather*}
        \phi_{w_{-i},b}(x) = \sigma(\innerproduct{w_{-i}}{x}+b) = \sigma(\innerproduct{w}{x_{-i}}+b)
        = \phi_{w,b}(x_{-i})
    \end{gather*}
    Suppose that $\phi_{w,b}(x)$ has the following Hermite decomposition:
    \begin{gather*}
        \phi_{w,b}(x) = \sum_{T\in\NN^d} \widehat{\phi}_{w,b}(T) \chi_T(x)
    \end{gather*}
    where $\chi_T(x) = \prod_{i=1}^d H_{t_i}(x_i)$. 
    Note that the Hermite polynomial $H_t(x_t)$ is an odd function for odd $t$ and even function for even $t$.
    Thus, we have
    \begin{gather*}
        \chi_T(x_{-i}) = 
        \begin{cases}
            \chi_T(x), & t_i \equiv 0 \; (mod \; 2) \\
            - \chi_T(x), & t_i \equiv 1 \; (mod \; 2)
        \end{cases}
    \end{gather*}
    Thus, for the function $\phi_{w_{-i},b}(x)$, we get:
    \begin{gather*}
        \phi_{w_{-i},b}(x)
        = \phi_{w,b}(x_{-i})
        = \sum_{T\in\NN^d, t_i \equiv 0 \; (mod \; 2)} \widehat{\phi}_{w,b}(T) \chi_T(x)
        - \sum_{T\in\NN^d,  t_i \equiv 1 \; (mod \; 2)} \widehat{\phi}_{w,b}(T) \chi_T(x)
    \end{gather*}
    which shows that 
    \begin{gather} \label{1}
        \widehat{\phi}_{w_{-i},b}(T) =
        \begin{cases}
            \widehat{\phi}_{w,b}(T), & t_i \equiv 0 \; (mod \; 2) \\
            - \widehat{\phi}_{w,b}(T), & t_i \not \equiv 0 \; (mod \; 2)
        \end{cases}
    \end{gather}
    Finally, consider $i\in[d]$ for which $t_i \not \equiv t_i' \; (mod \; 2)$.
    Then
    \begin{gather*}
        \mathbb{E}_{w,b}[\widehat{\phi}_{w,b}(T) \widehat{\phi}_{w,b}(T')]
        = \mathbb{E}_{w,b}[\widehat{\phi}_{w_{-i},b}(T) \widehat{\phi}_{w_{-i},b}(T')]
        = - \mathbb{E}_{w,b}[ \widehat{\phi}_{w,b}(T) \widehat{\phi}_{w,b}(T')]
    \end{gather*}
    Here, the first equality comes from the fact that $w$ and $w_{-i}$ have the same distribution, and the second equality comes from (\ref{1}).
    Hence, we obtained
    \begin{gather*}
        \mathbb{E}_{w,b}[\widehat{\phi}_{w,b}(T) \widehat{\phi}_{w,b}(T')]
        = - \mathbb{E}_{w,b}[ \widehat{\phi}_{w,b}(T) \widehat{\phi}_{w,b}(T')]
    \end{gather*}
    which completes the proof.
\end{proof}

\begin{proposition} \label{prop: 1}
    Let the random features model be trained in sparse setting and diverging $d$ regime with $\sigma(x) = (1+x)^2$ activation.
    Then it converges to the minimizer of:
    \begin{gather} \label{21}
        \sum_{i=1}^d \widehat{g}(2e_i)^2 \cdot \frac{d^2}{4} + \sum_{i<j} \widehat{g}(e_i+e_j)^2 \cdot \frac{d^2}{4} \\
        +\sum_{i=1}^d \widehat{g}(e_i)^2 \cdot \frac{d}{4} + \widehat{g}(0)^2 \cdot \frac{d}{6}
        + \sum_{i<j} \widehat{g}(2e_i) \widehat{g}(2e_j) \cdot \frac{d}{6} + \sum_{i=1}^d \widehat{g}(2e_i) \widehat{g}(0) \cdot \left(- \frac{\sqrt{2}}{3} d \right)
    \end{gather}
    over functions $g$ that interpolate the training data.
\end{proposition}

\begin{proof}
The arguments in Theorem \ref{theorem: 1} showing that the random feature model converges to the minimizer of $\widehat{g}^\top \Phi^{-1} \widehat{g}$ in small feature regime (where matrix $\Phi$ is defined in (\ref{2})) transfer directly to the sparse regime (but now we will have this convergence for fixed $d$ instead of fixed $\varepsilon$).
Let us compute this quadratic form explicitly.
    We have
    \begin{gather*}
        \phi_{w,b}(x) = \sigma(\innerproduct{w}{x} + b) = \left( \sum_{i=1}^d w_i x_i + b + 1 \right)^2 \\
        = \sum_{i=1}^d w_i^2 x_i^2 + 2 \sum_{i<j} w_i w_j x_i x_j + (b+1)^2 + 2 \sum_{i=1}^d w_i (b+1) x_i \\
        = \sum_{i=1}^d w_i^2 \sqrt{2} \frac{x_i^2-1}{\sqrt{2}} + 2 \sum_{i<j} w_i w_j x_i x_j + 2 \sum_{i=1}^d w_i (b+1) x_i + (b+1)^2 + \sum_{i=1}^d w_i^2
    \end{gather*}
    Thus the Hermite coefficients of the random feature $\phi_{w,b}$ are given by:
    \begin{gather*}
        \widehat{\phi}(2,0,\ldots,0) = w_1^2 \sqrt{2} \Rightarrow \mathbb{E}[\widehat{\phi}(2,0,\ldots,0)^2] = 2 \mathbb{E}[w_1^4] = \frac{6}{d^2} \\
        \widehat{\phi}(1,1,0,\ldots,0) = 2 w_1 w_2 \Rightarrow \mathbb{E}[\widehat{\phi}(1,1,0,\ldots,0)^2] = 4 \mathbb{E}[w_1^2 w_2^2] = \frac{4}{d^2} \\
        \widehat{\phi}(1,0,\ldots,0) = 2 w_1 (b+1) \Rightarrow \mathbb{E}[\widehat{\phi}(1,0,\ldots,0)^2] = 4 \mathbb{E}[w_1^2 (b+1)^2] = 4 \cdot \frac{1}{d} \cdot (1+\frac{1}{d}) = 4(\frac{1}{d} + \frac{1}{d^2}) \\
        \widehat{\phi}(0,0,\ldots,0) = (b+1)^2 + \sum_{i=1}^d w_i^2 \Rightarrow \mathbb{E}[\widehat{\phi}(0,0,\ldots,0)^2] \overset{(a)}{=} 4 + \frac{10}{d} + \frac{3}{d^2}
    \end{gather*}
    Here, (a) comes from
    \begin{gather*}
        \mathbb{E}[\widehat{\phi}(0,0,\ldots,0)^2]
        = \mathbb{E}\left[\left( (b+1)^2 + \sum_{i=1}^d w_i^2 \right)^2\right] \\
        = \mathbb{E}[(b+1)^4 + \sum_{i=1}^d w_i^4 + 2 \sum_{i<j} w_i^2 w_j^2 + 2 \sum_{i=1}^d w_i^2 (b+1)^2]
        = (1 + \frac{6}{d} + \frac{3}{d^2}) + d \cdot \frac{3}{d^2} + d(d-1) \frac{1}{d^2} + 2d \cdot \frac{1}{d}(1+\frac{1}{d}) \\
        = (1 + \frac{6}{d} + \frac{3}{d^2}) + \frac{3}{d} + (1 - \frac{1}{d}) + (2 + \frac{2}{d}) \\
        = 4 + \frac{10}{d} + \frac{3}{d^2}
    \end{gather*}
    For the cross-terms, we have
    \begin{gather*}
        \mathbb{E}[\widehat{\phi}(2,0,\ldots,0) \widehat{\phi}(0,2,\ldots,0)]
        = 2 \mathbb{E}[w_1^2 w_2^2] = \frac{2}{d} \\
        \mathbb{E}[\widehat{\phi}(2,0,\ldots,0) \widehat{\phi}(0,0,\ldots,0)]
        = \sqrt{2} \mathbb{E}[w_1^2 (b+1)^2 + w_1^4 + \sum_{i=2}^d w_1^2 w_i^2]
        = \sqrt{2} \left( \frac{1}{d} (1+\frac{1}{d}) + \frac{3}{d^2} + (d-1) \cdot \frac{1}{d^2} \right) \\
        = \sqrt{2} \left( \frac{2}{d} + \frac{3}{d^2} \right)
    \end{gather*}
    All other cross-terms equal to zero by Lemma \ref{lemma: 1}.
    Thus, matrix $\Phi$ is block-diagonal with the only non-unit block corresponding to the coefficients $(2,0,\ldots,0), (0,2,\ldots,0), (0,0,\ldots,2), (0,0,\ldots,0)$ ($d+1$ coefficient in the block in total).
    Thus, this $(d+1)\times(d+1)$ block takes the form:
    \begin{gather*}
        \begin{bmatrix}
            \frac{6}{d^2} & \frac{2}{d^2} & \ldots & \frac{2}{d^2} & \sqrt{2} \left( \frac{2}{d} + \frac{3}{d^2} \right) \\
            \frac{2}{d^2} & \frac{6}{d^2} & \ldots & \frac{2}{d^2} & \sqrt{2} \left( \frac{2}{d} + \frac{3}{d^2} \right) \\
            \vdots & \vdots & \ddots  & \vdots & \vdots \\
            \frac{2}{d^2} & \frac{2}{d^2} & \ldots & \frac{6}{d^2} & \sqrt{2} \left( \frac{2}{d} + \frac{3}{d^2} \right) \\
            \sqrt{2} \left( \frac{2}{d} + \frac{3}{d^2} \right) & \sqrt{2} \left( \frac{2}{d} + \frac{3}{d^2} \right) & \ldots & \sqrt{2} \left( \frac{2}{d} + \frac{3}{d^2} \right) & \left( 4 + \frac{10}{d} + \frac{3}{d^2} \right)
        \end{bmatrix}
    \end{gather*}
    Exploiting the permutation symmetry of the first $d$ Hermite coefficients in this matrix, we can search for its inverse in the following form:
    \begin{gather*}
        \begin{bmatrix}
            x & y & \ldots & y & z \\
            y & x & \ldots & y & z \\
            \vdots & \vdots & \ddots  & \vdots & \vdots \\
            y & y & \ldots & x & z \\
            z & z & \ldots & z & t
        \end{bmatrix}
    \end{gather*}
    From the condition that the product of the formal matrix with the later one must be $I_{d+1}$, we obtain the following 4 linear equations in 4 unknown variables:
    \begin{gather*}
        \frac{6}{d^2} x + 2 \left(\frac{1}{d} - \frac{1}{d^2}\right) y + \sqrt{2} \left( \frac{2}{d} + \frac{3}{d^2} \right) z = 1 \\
        \frac{2}{d^2} x + 2 \left(\frac{1}{d} + \frac{1}{d^2}\right) y + \sqrt{2} \left( \frac{2}{d} + \frac{3}{d^2} \right) z = 0 \\
        2 \left( \frac{1}{d} + \frac{2}{d^2} \right) z + \sqrt{2} \left( \frac{2}{d} + \frac{3}{d^2} \right) t = 0 \\
        \sqrt{2} \left( 2 + \frac{3}{d} \right) z + \left( 4 + \frac{10}{d} + \frac{3}{d^2} \right) t = 1
    \end{gather*}
    Solving this system, we obtain
    \begin{gather*}
        x = \frac{d^2}{4} + O(d) \\
        y = \frac{d}{12} + O(1) \\
        z = - \frac{\sqrt{2} d}{6} + O(1) \\
        t = \frac{d}{6} + O(1)
    \end{gather*}
    Combining with 
    \begin{gather*}
        \left( \mathbb{E}[\widehat{\phi}(1,1,0,\ldots,0)^2] \right)^{-1} = \frac{d^2}{4} \\
        \left( \mathbb{E}[\widehat{\phi}(1,0,\ldots,0)^2] \right)^{-1} = \frac{d}{4} + O(1)
    \end{gather*}
    we obtain:
    \begin{gather*}
        \widehat{g}^\top \Phi^{-1} \widehat{g}
        \approx \sum_{i=1}^d \widehat{g}(2e_i)^2 \cdot \frac{d^2}{4} + \sum_{i<j} \widehat{g}(e_i+e_j)^2 \cdot \frac{d^2}{4} \\
        +\sum_{i=1}^d \widehat{g}(e_i)^2 \cdot \frac{d}{4} + \widehat{g}(0)^2 \cdot \frac{d}{6}
        + \sum_{i<j} \widehat{g}(2e_i) \widehat{g}(2e_j) \cdot \frac{d}{6} + \sum_{i=1}^d \widehat{g}(2e_i) \widehat{g}(0) \cdot \left(- \frac{\sqrt{2}}{3} d \right)
    \end{gather*}
\end{proof}

\begin{proof}[Proof of Example \ref{example: 1}]
    Assume that $g$ is an interpolator of the training data. We show that
\begin{equation*}
    g(x) = \widehat{g}(0) + \widehat{g}(e_1) x_1 + \widehat{g}(2e_1)\frac{x_1^2-1}{\sqrt{2}} + \sum_{i \geq 2} \left(\widehat{g}(e_i) x_i + \widehat{g}(e_1 + e_i) x_1 x_i \right) \, ,
\end{equation*}
with the constraints that $\widehat{g}(0) + \widehat{g}(e_1) = 1$ and $\widehat{g}(e_i) + \widehat{g}(e_1 + e_i) = 0$ for all $i \geq 2$. 

Recall that the support of this distribution contains a subset of the form $\{1\} \times S_2 \times \dots \times S_d$ where $S_2, \dots, S_d$ have cardinality at least $3$. We apply Theorem 1.1 of \cite{alon1999combinatorial}. There exists multivariate polynomials $h_1, h_2, \dots, h_d$ such that $g = \sum_i h_i g_i$ with $g_1(x) = x_1-1$ and for $i \geq 2$, $g_i(x) = \Pi_{s \in S_i}(x-x_i)$. Moreover, the degree of $h_1$ is at most $\deg f - \deg g_1 = 1$ and the degree of $h_i$ ($i\geq 2$) is at most $\deg f - \deg g_i = -1$. Thus $h_2 = \dots = h_d = 0$. We thus get $g(x) = h_1(x) (x_1 -1)$ with $h_1(x)$ affine, which is equivalent to the above statement. 

To sum up, we minimize 
\begin{align*}
            \widehat{g}^\top \Phi^{-1} \widehat{g}
        &\approx \widehat{g}(2e_1)^2 \cdot \frac{d^2}{4} + \sum_{i \geq 2} \widehat{g}(e_1+e_i)^2 \cdot \frac{d^2}{4}  + \widehat{g}(e_1)^2 \cdot \frac{d}{4}\\
        &\qquad+\sum_{i\geq 2} \widehat{g}(e_i)^2 \cdot \frac{d}{4} + \widehat{g}(0)^2 \cdot \frac{d}{6}
        +  \widehat{g}(2e_1) \widehat{g}(0) \cdot \left(- \frac{\sqrt{2}}{3} d \right)
\end{align*}
in the variables $\widehat{g}(0), \widehat{g}(e_1), \widehat{g}(2e_1), \widehat{g}(e_1 + e_i), \widehat{g}(e_i)$, $i \geq 2$ with constraints $\widehat{g}(0) + \widehat{g}(e_1) = 1$ and $\widehat{g}(e_i) + \widehat{g}(e_1 + e_i) = 0$. 

This optimization problem is separable in the groups of variables $\{\widehat{g}(0), \widehat{g}(e_1), \widehat{g}(2e_1)\}$ and $\{\widehat{g}(e_1 + e_i), \widehat{g}(e_i)\}$ (both the constraints and the objective are separable). The second optimization problem is obvious and leads to the unique solution $\widehat{g}(e_i) = \widehat{g}(e_1 + e_i) = 0$, $i \geq 2$. Simplifying the scaling of the objective, we are left with the optimization problem of minimizing 
\begin{align*}
             \widehat{g}(2e_1)^2 \cdot \frac{d}{4}  + \widehat{g}(e_1)^2 \cdot \frac{1}{4}+ \widehat{g}(0)^2 \cdot \frac{1}{6}
        +  \widehat{g}(2e_1) \widehat{g}(0) \cdot \left(- \frac{\sqrt{2}}{3}  \right)
\end{align*}
under the constraint $\widehat{g}(0) + \widehat{g}(e_1) = 1$. 

Minimizing marginally in $\widehat{g}(2e_1)$, we obtain that $\widehat{g}(2e_1) = \widehat{g}(0) \frac{2\sqrt{2}}{3d}$. Substituting in the expression above, we minimize 
\begin{align*}
              \widehat{g}(e_1)^2 \cdot \frac{1}{4}+ \widehat{g}(0)^2 \cdot \frac{1}{6}
        -   \widehat{g}(0)^2  \frac{4}{9d} 
\end{align*}
under the constraint $\widehat{g}(0) + \widehat{g}(e_1) = 1$. The last term has a negligible effect as $d \to \infty$ and thus the solution converges to the solution with $\widehat{g}(0) = \frac{3}{5}$, $\widehat{g}(e_1) = \frac{2}{5}$.
    Thus, the random feature model learns the function $f_{\textnormal{RF}}(x) = \frac{2}{5} x_1 + \frac{3}{5}$.
\end{proof}

\section{Proof of Theorem \ref{thm:roots-unity}} 
\label{ap:roots-unity}

The proof follows a structure similar to the one of \cite{gotu} (and to the one of Section \ref{sec:main} and Appendix \ref{sec:lemmas}): the strategy is to study the covariance matrix of the random features. The minimum degree bias follows from different scales (in $d$) of different polynomial components of the random features. Here we only outline the main differences with the previous proofs. 

For functions $h:\U_n^d \to \CC$, the appropriate decomposition is given by its discrete Fourier transform. It corresponds to the linear decomposition on the basis of monomials 
\begin{equation*}
    \chi_{j_1, \dots, j_d}(x) = x_1^{j_1} \cdots x_d^{j_d} \, .
\end{equation*}
This basis is orthonormal in the Hermitian space $L^2(\U_n^d, \Unif(\U_n^d))$. More concretely, the discrete Fourier transform of $h:\U_n^d \to \CC$ is $\widehat{h}:\{0, \dots, n-1\}^d \to \CC$, such that 
\begin{gather*}
    \widehat{h}(j_1, \dots, j_d) = \E_{x \sim \Unif(\U_n^d)}\left[h(x) \overline{x}_1^{j_1} \cdots \overline{x}_d^{j_d}\right] \, , \\
    j_1, \dots, j_d \in \{0, \dots, n-1\} \, .
\end{gather*}
The inverse Fourier transform states that 
\begin{equation*}
    h(x) = \sum_{j_1, \dots, j_d \in \{0, \dots, n-1\}}  \widehat{h}(j_1, \dots, j_d) x_1^{j_1} \cdots x_d^{j_d} \, .
\end{equation*}
We consider the discrete Fourier transform of the random feature $\phi_{w,b}(x) = \sigma\left(\langle w, x \rangle + b\right)$: 
\begin{equation*}
    \widehat{\phi}_{w,b}(j) = \E_x\left[{\phi}_{w,b}(x) \overline{x}_1^{j_1} \cdots \overline{x}_d^{j_d}\right] \, .
\end{equation*}
Theorem \ref{thm:roots-unity} stems from the following proposition.

\begin{proposition}
    Consider $j, j' \in \{0, \dots, n-1\}^d$, $j \neq j'$. Then 
    \begin{enumerate}
        \item $\E_{w,b}\left[\widehat{\phi}_{w,b}(j) \overline{\widehat{\phi}_{w,b}(j')}\right] = 0$, and
        \item $\E_{w,b}\left[\left\vert \widehat{\phi}_{w,b}(j) \right\vert^2\right] = \Theta\left(d^{-\vert j \vert}\right)$ as $d \to \infty$. 
    \end{enumerate}
\end{proposition}

The two points of this proposition correpond respectively to the points A4 and A3 in Lemma A.1 of \cite{gotu}.

\begin{proof}
        \begin{enumerate}
            \item From $j \neq j'$, we know that there exists $l \in \{1, \dots, d \}$ such that $j_l \neq j'_l$. Let $R$ denote the rotation of the $l$-th root of unity 
            \begin{gather*}
                T: (x_1, \dots, x_d) \in \U_n^d \mapsto \\ (x_1, \dots, x_{l-1}, e^{i\frac{2\pi}{l}} x_l, x_{l+1}, \dots, x_n) \, .
            \end{gather*}
            We compute the effect of a rotation of $w$ on the discrete Fourier transform of a random feature:
            \begin{align*}
                \widehat{\phi}_{Tw,b}(j) = \E_x\left[{\phi}_{Tw,b}(x) \overline{x}_1^{j_1} \cdots \overline{x}_d^{j_d}\right] \, .
            \end{align*}
            Here we note that the uniform distribution of $\U_n^d$ is invariant under the map $T$, thus 
                        \begin{align*}
                \widehat{\phi}_{Tw,b}(j) = \E_x\left[{\phi}_{Tw,b}(Tx) \overline{(Tx)}_1^{j_1} \cdots \overline{(Tx)}_d^{j_d}\right] \, .
            \end{align*}
            Moreover, 
            \begin{align*}
                \langle Tw, Tx \rangle & = \overline{w}_1 x_1 + \dots + \overline{w}_{l-1} x_{l-1} \overline{e^{i\frac{2\pi}{n}} w_l} e^{i\frac{2\pi}{n}} x_l \\ 
                & + \overline{w}_{l+1} x_{l+1} + \dots + \overline{w}_d w_d = \langle w, x \rangle
            \end{align*}
            and thus ${\phi}_{Tw,b}(Tx) = {\phi}_{w,b}(x)$. 
        As a consequence, we have 
        \begin{align*}
                \widehat{\phi}_{Tw,b}(j) & = \E_x\left[{\phi}_{w,b}(x) \overline{(Tx)}_1^{j_1} \cdots \overline{(Tx)}_d^{j_d}\right] \\
                & = e^{-i \frac{2\pi}{n}} \E_x\left[{\phi}_{w,b}(x) \overline{x}_1^{j_1} \cdots \overline{x}_d^{j_d}\right]  \\
                & = e^{-i \frac{2\pi j_l}{n}} \widehat{\phi}_{w,b}(j) \, .
            \end{align*}
            We are ready to conclude. The distribution of $w$ is invariant under the map $T$, thus 
            \begin{gather*}
                \E_{w,b}\left[\widehat{\phi}_{w,b}(j) \overline{\widehat{\phi}_{w,b}(j')}\right] = \E_{w,b}\left[\widehat{\phi}_{Tw,b}(j) \overline{\widehat{\phi}_{Tw,b}(j')}\right] \\
                = e^{i\frac{2\pi(j_l'-j_l)}{n}} \E_{w,b}\left[\widehat{\phi}_{w,b}(j) \overline{\widehat{\phi}_{w,b}(j')}\right] \, .
            \end{gather*}
            As $j_l \neq j_l'$ and $j_l, j_l' \in \{0, \dots, n-1\}$, we have $e^{i\frac{2\pi(j_l'-j_l)}{n}} \neq 1$. Thus it must be that $\E_{w,b}\left[\widehat{\phi}_{w,b}(j) \overline{\widehat{\phi}_{w,b}(j')}\right] = 0$.
            \item We make a Taylor expansion of $\sigma$ at $0$:
            \begin{align*}
                \widehat{\phi}_{w,b}(j) = \E_x\left[\sigma(\langle w,x\rangle + b) \overline{\chi_j(x)} \right] = \sum_{k=0}^\infty \frac{\sigma^{(k)}(0)}{k!} \E_x\left[(\langle w,x\rangle + b)^k \overline{\chi_j(x)} \right] \, .
            \end{align*}
            We make three cases depending on the index $k$ of the sum:
            \begin{itemize}
                \item If $k < |j|$, then $(\langle w,x\rangle + b)^k$ is a polynomial of degree $<k\leq |j| = \deg \chi_j$ thus by orthogonality $\E_x\left[(\langle w,x\rangle + b)^k \overline{\chi_j(x)} \right] =0$.
                \item If $k = |j|$, then 
                \begin{align*}
                    \E_x\left[(\langle w,x\rangle + b)^k \overline{\chi_j(x)} \right] &= \E_x\left[(\overline{w}_1 x_1 + \dots + \overline{w}_d x_d + b)^k \overline{\chi_j(x)} \right] \\
                    &= \sum_{l_1 + \dots + l_d + l_{d+1} = k} {k \choose l_1, \dots, l_d, l_{d+1}} \E_x\left[(\overline{w}_1 x_1)^{l_1} \cdots (\overline{w}_d x_d)^{l_d} b^{l_{d+1}}\overline{\chi_j(x)} \right] \\
                    &= \sum_{l_1 + \dots + l_d + l_{d+1} = k} {k \choose l_1, \dots, l_d, l_{d+1}} \overline{w}_1^{l_1} \cdots \overline{w}_d^{l_d}  b^{l_{d+1}}\E_x\left[x_1^{l_1-j_1} \cdots x_d^{l_d-j_d}  \right] \, .
                \end{align*}
                Note that $\E_x\left[x_1^{l_1-j_1} \cdots x_d^{l_d-j_d}  \right]$ equals $1$ if $l_1 \equiv j_1 \mod n, \dots, l_d \equiv j_d \mod n$ and $0$ otherwise. As $l_1 + \dots + l_d = k-l_{d+1} \leq k = j_1 + \dots + j_d$, this is possible if and only if $l_1 = j_1$, \dots, $l_d = j_d$, $l_{d+1} = 0$. Thus 
            \begin{align*}
                    \E_x\left[(\langle w,x\rangle + b)^k \overline{\chi_j(x)} \right] &= {k \choose j_1, \dots, j_d} \overline{w}_1^{j_1} \cdots \overline{w}_d^{j_d}  = \frac{1}{d^{|j|/2}} {k \choose j_1, \dots, j_d}\overline{u}_1^{j_1} \cdots \overline{u}_d^{j_d}
                \end{align*}
                where $u := d^{1/2} w$ (and thus $u_1, \dots, u_d$ are i.i.d. with standard Gaussian distribution in the complex plane). 
                \item If $k > |j|$, then 
                \begin{align*}
                     \E_x\left[(\langle w,x\rangle + b)^k \overline{\chi_j(x)} \right] = \frac{1}{d^{k/2}} \E_x\left[(\langle u,x\rangle + b)^k \overline{\chi_j(x)} \right] \, ,
                \end{align*}
                where again $u = d^{1/2} w$ and $c := d^{1/2} b$ (and thus $b$ has standard Gaussian distribution in the complex plane).
                \end{itemize}
                            Putting these three points together, we obtain 
                \begin{align*}
                \widehat{\phi}_{w,b}(j) = \frac{1}{d^{|j|/2}} {k \choose j_1, \dots, j_d}\overline{u}_1^{j_1} \cdots \overline{u}_d^{j_d} + o\left(\frac{1}{d^{|j|/2}}\right)  \, .
            \end{align*}
            Thus 
            \begin{align*}
                \E_{w,b} \left[\left\vert \widehat{\phi}_{w,b}(j) \right\vert^2\right] = \frac{1}{d^{|j|}} {k \choose j_1, \dots, j_d}^2 \E_w \left[\vert {u}_1 \vert^{2j_1} \cdots \vert {u}_d \vert^{2j_d} \right] + o\left(\frac{1}{d^{|j|}}\right) = \Theta\left(\frac{1}{d^{|j|}}\right) \, .
            \end{align*}
            \end{enumerate}
    \end{proof}


\end{document}